\documentclass[10pt,twocolumn,letterpaper]{article}

\usepackage[pagenumbers]{cvpr} 
\usepackage{url}            
\usepackage{booktabs}       
\usepackage{amsfonts}       
\usepackage{nicefrac}       
\usepackage{microtype}      
\usepackage{xcolor}         
\usepackage{graphicx}
\usepackage{epstopdf}
\usepackage{amsmath}
\usepackage{subcaption}
\usepackage{algorithm}
\usepackage{algpseudocode}
\usepackage{amsmath}
\usepackage{multirow}
\usepackage{makecell}
\usepackage{utfsym}
\usepackage{bm}
\usepackage{amssymb, amsthm}
\newtheorem{theorem}{Theorem}[section]

\definecolor{cvprblue}{rgb}{0.21,0.49,0.74}
\usepackage[pagebackref,breaklinks,colorlinks,allcolors=cvprblue]{hyperref}

\def\paperID{*****} 
\def\confName{CVPR}
\def\confYear{2026}

\title{GFPL: Generative Federated Prototype Learning for Resource-Constrained and Data-Imbalanced Vision Task}

\author{  Shiwei Lu$^{1,2}$, Yuhang He$^{1}$, Jiashuo Li$^{1}$, Qiang Wang$^{1}$, Yihong Gong$^{1}$ \thanks{Corresponding author}\\    
	$^1$Xi'an Jiaotong University\\ $^2$ Air Force Engineering University \\
}

\begin{document}
\maketitle
\begin{abstract}
Federated learning (FL) facilitates the secure utilization of decentralized images, advancing applications in medical image recognition and autonomous driving. However, conventional FL faces two critical challenges in real-world deployment: ineffective knowledge fusion caused by model updates biased toward majority-class features, and prohibitive communication overhead due to frequent transmissions of high-dimensional model parameters. Inspired by the human brain’s efficiency in knowledge integration, we propose a novel Generative Federated Prototype Learning (GFPL) framework to address these issues. Within this framework, a prototype generation method based on Gaussian Mixture Model (GMM) captures the statistical information of class-wise features, while a prototype aggregation strategy using Bhattacharyya distance effectively fuses semantically similar knowledge across clients. In addition, these fused prototypes are leveraged to generate pseudo-features, thereby mitigating feature distribution imbalance across clients. To further enhance feature alignment during local training, we devise a dual-classifier architecture, optimized via a hybrid loss combining Dot Regression and Cross-Entropy. Extensive experiments on benchmarks show that GFPL improves model accuracy by 3.6\% under imbalanced data settings while maintaining low communication cost. 
\end{abstract}    
\section{Introduction}
\label{sec:introduction}

The proliferation of Internet of Things ($\text{IoT}$) devices generates massive multi-sensor heterogeneous data streams, enabling environment-aware intelligence via deep learning. However, privacy concerns and prohibitive communication costs create fragmented data silos, preventing centralized curation. Federated Learning ($\text{FL}$) emerges as a decentralized alternative, coordinating distributed model training through parameter exchange rather than raw data sharing. Despite preliminary success in $\text{IoT}$ applications \cite{kou2025fast,zhang2024survey,arunan2023federated}, two persistent challenges hinder practical $\text{FL}$ deployment: (1) ineffective knowledge fusion under prevalent imbalanced and non-IID data distributions \cite{dai2023tackling,zhou2023fedfa}, and (2) inefficient training caused by parameter-various model interactions \cite{xiong2023feddm,crawshaw2024federated}. Specifically, traditional parameter aggregation induces gradient conflicts under data shifts, compromising global model convergence, and frequent high-dimensional parameter transmission severely strains resource-constrained $\text{IoT}$ environments. These dual constraints fundamentally limit current $\text{FL}$ frameworks in real-world deployments.

Current approaches to federated knowledge integration, such as regularization constraints \cite{yuan2022convergence} and knowledge distillation\cite{zhang2022fine}, attempt to mitigate parameter divergence through model alignment or logit-based knowledge transfer. However, these methods either incur prohibitive communication overhead from continuous parameter transmission\cite{lu2023top,kim2024spafl}  or introduce secondary data collection challenges for output space calibration\cite{hong2024fedavp}. 
Drawing inspiration from generative learning\cite{cao2022perfed} and prototype learning\cite{tan2022fedproto}, we propose Generative Federated Prototype Learning (GFPL) – a communication-efficient paradigm that achieves cross-client knowledge fusion through prototype interaction while eliminating model parameter transmission. The core insight stems from two biological analogies:

\textbf{Distributed concept refinement: }Mirroring human cognitive mechanisms where object concepts evolve through prototype interaction (e.g., forming the enriched prototype "blue-small/large-wheeled" for "car" by merging local prototypes from "sky blue-small-wheeled" and "ocean blue-large-wheeled", GFPL enhances knowledge fusion effects via prototype generation and interaction.

\textbf{Generative concept augmentation:}Emulating the human ability to instantiate abstract prototypes as concrete mental images, GFPL leverages prototype-driven generation method to convert global class prototypes into diversified pseudo-features across client devices, mimicking neurocognitive materialization processes.

Our main contributions are summarized as follows:

(1) We devised a Gaussian Mixture Model-based prototype generation mechanism and a Bhattacharyya distance-driven prototype fused method in visual FL to ensure secure and effective cross-client information interaction with minimal communication overhead

(2) To address the challenge of feature shift caused by imbalanced data, we designed a dual-classifier structure and a hybrid loss function for local training. This design synchronously aligns distributed features with both predefined vectors and class labels, which effectively enhances both the global consistency of intra-class features and the global separability of inter-class features.

(3) To enhance model generalization, we developed a pseudo-feature generation method to retrain the projection layer within the dual-classifier structure. By integrating all proposed components, a novel Generative Federated Prototype Learning (GFPL) framework and algorithm were proposed. Extensive experiments demonstrate that GFPL achieves the highest average test accuracy with the lowest communication cost.

\section{Related Work}
\label{sec:formatting}
\subsection{The knowledge fusion of imbalanced data in FL}
IoT-generated data distributions exhibit significant heterogeneity due to diverse preferences and collection environments of clients, leading to substantial parameter divergence between locally trained models in FL and hindering effective aggregation of the global model. In FL, knowledge fusion from imbalanced data seeks to distill distributed knowledge into a unified global model, which has better generalization capabilities than local models trained by clients. Recent advances address this challenge through three primary approaches:
(1) Data-level: mitigate source imbalance through distribution enhancement, including data augmentation/synthesis \cite{hong2024fedavp,cao2022perfed} and self-supervised pretraining \cite{li2021model};
(2) Model-level: adapt to data heterogeneity via structural/training innovations, featuring model personalization \cite{tan2022fedproto,tan2022towards}, dynamic parameter pruning \cite{yang2023dynamic}, knowledge distillation \cite{zhang2022fine,yao2023fedgkd}, and feature decoupling-alignment \cite{zhou2023fedfa};
(3) Framework-level: strengthen system robustness through architectural improvements, such as adaptive optimization \cite{karimireddy2020scaffold}, meta learning \cite{yang2023personalized}, client clustering \cite{ghosh2020efficient}, and hybrid FL paradigms \cite{liu2024fedcd}.
\subsection{Prototype learning}
Rooted in cognitive psychology, prototype theory posits that humans abstract category concepts through typical feature aggregation-a mechanism mirroring feature extraction processes of deep learning \cite{snell2017prototypical,zhou2023revisiting}. As an interpretable representation paradigm, prototype learning constructs abstract representations of typical samples, demonstrating cross-domain knowledge representation capabilities \cite{huang2023rethinking}, such as category centroids via feature vector averaging in Computer Vision (CV) field \cite{tan2022fedproto,zhang2024fedtgp}, semantic intent prototypes through query clustering in dialogue systems of Natural Language Processing (NLP) \cite{tu2023bag,he2020learning}, and heterogeneous data alignment bridges \cite{wang2022cross,liu2021adaptive} in Cross-Modal Learning, exemplified by contrastive image-text prototype space of CLIP \cite{luddecke2022image}.
In FL, FedProto \cite{tan2022fedproto} pioneers prototype integration by aggregating class centroids of features  into class-specific global prototypes. However, weighted averaging of heterogeneous client prototypes risks dominance by low-quality instances and limits class representation fidelity. Thus, enhancing prototype robustness becomes critical for addressing data imbalance in FL systems.
\subsection{Generative learning}
Generative learning addresses data scarcity and class imbalance by modeling data distributions to synthesize samples/features, enhancing model generalization through implementations like few-shot learning with pseudo-sample generators trained on base-class data \cite{wang2020generalizing} and novel-class generation via semantic embeddings \cite{xu2022generating}. Breakthroughs in generative models, including StyleGAN \cite{karras2019style}, GPT \cite{mann2020language}, and DeepSeek \cite{guo2025deepseek}, have enabled high-fidelity image/text synthesis. In FL, these techniques facilitate knowledge sharing via server-side pretrained generators \cite{zhang2024upload,zhu2021data} or local knowledge distillation for task-relevant image construction \cite{jiang2024fedpa}. However, current approaches face limitations, such as dependency on auxiliary datasets, conflicts between model transmission and communication overhead, and potential privacy risks from synthetic data.
\section{Problem setting}
\label{sec:problem_setting}
\subsection{Knowledge fusion of imbalanced data in federated learning }
In FL with $m$ clients, each client possesses a private dataset $\mathbb{D}_i$ drawn from a distribution $\mathbb{P}_{i}(x,y)$, where $x,y$ denotes an input sample and its labels. During initialization of a supervised learning task, clients download a unified global model $\mathcal{F}(\omega,\cdot)$ from the server, which comprises a feature extractor $f(\omega_1,\cdot )$ and a classifier $g(\omega_2,\cdot )$ , and $\omega=(\omega_1,\omega_2)$ is the model parameters (e.g. Convolutional Neural Networks(CNN) model). For an input $x$, the feature embedding $f(\omega_1,x) \in \mathbb{R}^{d} $ is processed by the classifier to generate predictions $g(\omega_2,f(\omega_1,x))$, where $d$ is the dimension of flatten feature. Clients locally optimize their models via loss function $\mathcal{L}_{local}$ (e.g., cross-entropy) with $\left \{ \mathbb{D}_i \right \}_{i \in [m]}$.  The classical FedAvg framework \cite{mcmahan2017communication} aims to minimize:
\begin{equation}\label{con:eq_1} \small
	\mathcal{L}_{gloal} = \underset{\omega}{\arg \min } \sum_{i=1}^{m} \frac{\left|\mathbb{D}_{i}\right|}{N} \mathcal{L}_{local}(\mathcal{F}(\omega ; x), y),
\end{equation}
where $\mathit{N}$ is the total samples across clients. Generally, FedAvg can obtain an effective global model when distributions $\left \{\mathbb{P}_{i}(x,y)\right \}_{i \in [m]}$ are Independent and Identically Distributed(IID).  However, in reality each client collects data from different scenarios, which makes $\mathbb{P}_{i}$ varies between clients and exacerbates global imbalance of FL. In addition, resource-constrained devices (e.g., sensors, smartphones) suffer from limited computation, storage, and communication bandwidth. While conventional methods \cite{zhou2023fedfa,mcmahan2017communication,li2020federated} improve knowledge fusion via model-weighted aggregation, they incur prohibitive communication overhead and storage costs. To address these challenges, we propose a dual-classifier structure for local training and a novel prototype generation and interaction method, where the former aims to align intra-class features of $\left \{\mathbb{D}_i \right \}_{i \in [m]}$ and the latter is designed to enhance model generalization by pseudo-feature generation.
\subsection{Classifier with equiangular tight frame (ETF)}
Drawing from Neural Collapse(NC) theory \cite{papyan2020prevalence}, revealing that features under IID data conditions collapse into symmetric geometric structures, where class embeddings become maximally separable and equidistant. Inspired by this theory, we introduce a predefined ETF classifier to enforce alignment of distributed features to the corresponding ETF vectors. Furthermore, intra-class consistency and inter-class discriminability of features can be enhanced, effectively addressing challenges such as feature/label shift across clients.

\textbf{Definition (simple equiangular tight frame)}: for a classification task with $K$ classes, its ETF classifier can be designed by:
\begin{equation}\label{con:eq_2} \small
	\mathbf{M}=\sqrt{\frac{K}{K-1}} \mathbf{U}\left(\mathbf{I}_{K}-\frac{1}{K} \mathbf{1}_{K} \mathbf{1}_{K}^{T}\right),
\end{equation}
where $\mathbf{M}=\left[\mathbf{m}_{1}, \ldots, \mathbf{m}_{K}\right] \in \mathbb{R}^{d \times K}, \mathbf{U} \in \mathbb{R}^{d \times K}$ is a matrix where column vectors are orthogonal and satisfy  $\mathbf{U}^\mathit{T}\mathbf{U}=\mathbf{I}_{K}$, $\mathbf{I}_{K}$ is the identity matrix, and $\mathbf{1}_{K}$ is an all-ones vector.
Each vector in the simple ETF has an equal $\ell_{2}$ norm and the same pair-wise angle, i.e.,
\begin{equation}\label{con:eq3} \small
	\mathbf{m}_{i}^{T} \mathbf{m}_{j}=\frac{K}{K-1} \delta_{i, j}-\frac{1}{K-1}, \forall i, j \in[1, \ldots, K],
\end{equation}
where $\delta_{i, j}=1$ when $i=j$ and 0 otherwise.The pair-wise angle $-\frac{1}{K-1}$ is the maximal equiangular separation of $K$ vector in  $\mathbb{R}^{d}$.
\subsection{Gaussian Mixture Model (GMM)}
GMM\cite{reynolds2009gaussian} is a probability-based generative model for modeling complex data distributions. Theoretically, GMM can fit arbitrarily complex distributions. The core idea is to approximate an arbitrary continuous probability density function through a linear combination of multiple Gaussian distributions.
Generally, a GMM model consists of $n$ Gaussian distributions (or components), and the parameters of each component include the mean $\mu_{i}$, covariance matrix $\Sigma_{i}$ and weight $\pi_i $. The probability density function of GMM is:
\begin{equation}\label{con:eq_4} \small
	p(\boldsymbol{x})=\sum_{i=1}^{n} \pi_{i} \cdot \mathcal{N}\left(\boldsymbol{x} \mid \boldsymbol{\mu}_{i}, \boldsymbol{\Sigma}_{i}\right),
\end{equation}
where $\pi_i\ge0 $, $\sum_{i=1}^{n} \pi_{i}=1$, and $\mathcal{N}$ represents the Gaussian distribution. GMM parameters are estimated by an Expectation-Maximization (EM) algorithm, which is divided into two steps,

\textbf{E step:} Calculate the posterior probability with Eq. \Ref{con:eq_5} that data point $\boldsymbol{x}_j$ belongs to $i$-th component
\begin{equation}\label{con:eq_5} \small
	\gamma_{i j}=\frac{\pi_{i} \mathcal{N}\left(\boldsymbol{x}_{j} \mid \boldsymbol{\mu}_{i},  \boldsymbol{\Sigma}_{i}\right)}{\sum_{k=1}^{n} \pi_{k} \mathcal{N}\left(\boldsymbol{x}_{j} \mid \boldsymbol{\mu}_{k}, \boldsymbol{\Sigma}_{k}\right)} .
\end{equation}

\textbf{M step:} According to $\gamma_{i j}$ update parameters:
\begin{equation}\label{con:eq_6} \small
	\begin{split}
	\pi_i=\frac{N_i}{N},
	\quad\boldsymbol{\mu}_i=\frac{\sum_{i=1}^N\gamma_{ij}\boldsymbol{x}_j}{N_i}, \\ \quad\boldsymbol{\Sigma}_i=\frac{\sum_{i=1}^N\gamma_{ij}(\boldsymbol{x}_j-\boldsymbol{\mu}_i)(\boldsymbol{x}_j-\boldsymbol{\mu}_i)^\top}{N_i},
	\end{split}
\end{equation}
where $N_i=\sum_{i=1}^N\gamma_{ij}$. The EM algorithm needs to be iterated until it converges. We employ GMM to model the feature distributions within the same class, thereby obtaining statistical information of intra-class features.
\subsection{Bhattacharyya Distance}
The Bhattacharyya Distance is for quantifying the similarity between two probability distributions. For two univariate Gaussian distributions $\mathcal{N}(\mu_1, \sigma_1^2)$ and $\mathcal{N}(\mu_2, \sigma_2^2)$, the Bhattacharyya Distance has a closed-form solution:
\begin{equation} \small
	D_B = \frac{1}{4} \frac{(\mu_1 - \mu_2)^2}{\sigma_1^2 + \sigma_2^2} + \frac{1}{2} \ln \left( \frac{\sigma_1^2 + \sigma_2^2}{2\sigma_1\sigma_2} \right)
\end{equation}

For multivariate Gaussian distributions $\mathcal{N}(\bm{\mu}_1, \bm{\Sigma}_1)$ and $\mathcal{N}(\bm{\mu}_2, \bm{\Sigma}_2)$, the Bhattacharyya Distance is given by:
\begin{equation} \small
	D_B = \frac{1}{8} (\bm{\mu}_1 - \bm{\mu}_2)^\top \bm{\Sigma}^{-1} (\bm{\mu}_1 - \bm{\mu}_2)	+ \frac{1}{2} \ln \left( \frac{\det \bm{\Sigma}}{\sqrt{\det \bm{\Sigma}_1 \det \bm{\Sigma}_2}} \right)
\end{equation}
where $\bm{\Sigma} = \frac{\bm{\Sigma}_1 + \bm{\Sigma}_2}{2}$. By setting a threshold $C$, components from different GMMs can be merged.

\section{Methodology}
\label{sec:methodology}
We present a\textbf{ Generative Federated Prototype Learning (GFPL)} framework addressing ineffective knowledge fusion and prohibitive communication overhead in FL through two innovative components: 1) A \textbf{Dual-Classifier Structure (DCS)} for cross-client feature alignment, and 2) A \textbf{Pseudo Feature Generation (PFG)} mechanism enhanced by our prototype generation and  interaction protocol. As illustrated in Fig.\ref{fig:gfpl}, consider a FEMNIST image classification task \cite{caldas2018leaf} with skewed digit distributions across clients, each client derives class-wise prototypes via GMM from local features. 
\begin{figure*}[htbp]
	\centering 
	\fbox{\includegraphics[scale=0.35]{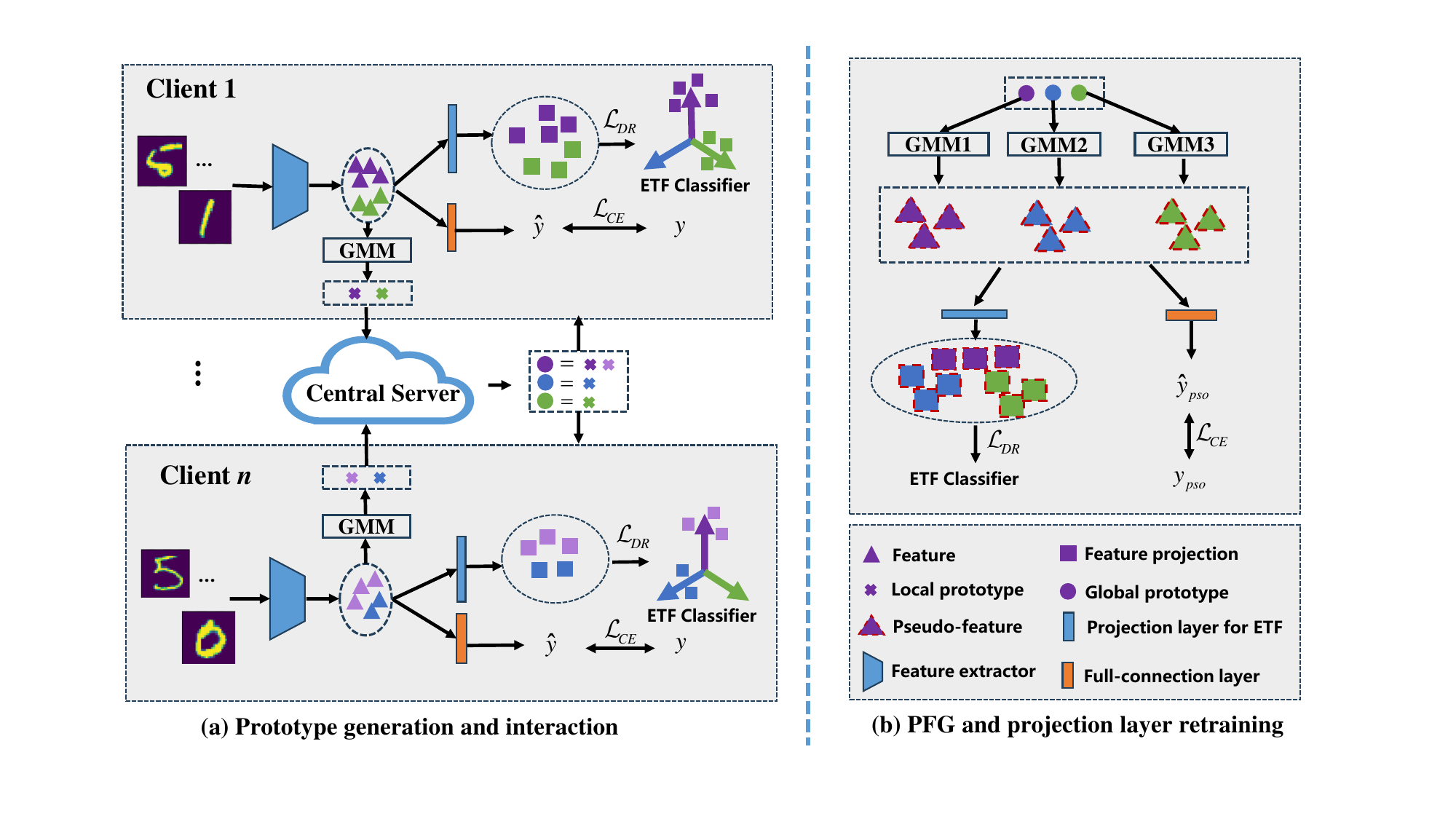} } 
	\caption{Generative Federated Prototype Learning (GFPL):(a) the client  trains feature extractor, projection layers through local data, while generating local prototypes with GMM and uploading them to the server for prototype interaction; (b) After receiving global prototypes, the client replaces local prototypes with global prototypes and further generates balanced pseudo features through GMM sampling, thereby retraining the model projection layer.}
	\label{fig:gfpl}
\end{figure*}
The server aggregates overlapping prototypes (e.g., digit 5 from multiple clients) through concatenation while preserving unique ones (e.g., digits 0/1) to form global prototypes. These unified global prototypes enable clients to generate balanced pseudo-features for underrepresented classes, effectively mitigating ineffective knowledge fusion from imbalanced data through projection layer retraining, while maintaining communication efficiency via lightweight prototype transmission.

\subsection{Feature Alignment with the Dual-Classifier Structure  (DCS)} \label{section41}
In data-imbalanced FL scenarios, the difficulty in cross-client feature alignment has resulted in ineffective knowledge fusion. Although federated prototype learning (Fedproto) \cite{tan2022fedproto} addresses this by aligning local features to global  prototypes (class centroids), its reliance on frequent prototype interaction incurs communication overhead. Thus, such a question arises: \textbf{can we achieve cross-client feature alignment without incurring communication overhead?}

Building on neural collapse theory in imbalanced learning \cite{yang2022inducing}, we address above problem in FL where a predefined Equiangular Tight Frame (ETF) classifier is designed to enhance intra-class consistency and inter-class separability of cross-client features. Although recent work \cite{li2023no} integrated ETF into FL for knowledge fusion through model aggregation, none explore their synergy with prototype-based methods. Thus, we propose a dual-classifier structure to embed ETF into CNN, coupled with a hybrid loss function for local training, which mitigates feature misalignment of imbalanced data without incurring communication overhead.

\textbf{Construction of ETF classifier:}  For a $K$ classification task, we first generate a random matrix  $\mathbf{A}\in \mathbb{R}^{d\times K} $, where $d$ is the dimension of feature and $d>K$. Afterwards, compute the reduced QR decomposition of the matrix $\mathbf{A}$:
\begin{equation}\label{con:eq_7} \small
	\mathbf{A} = \mathbf{Q}\mathbf{R},
\end{equation}
where $\mathbf{Q}=[\mathbf{q}_{1},\mathbf{q}_{2},...,\mathbf{q}_{K}]\in \mathbf{R}^{d\times K} $ is a matrix where column vectors are orthogonal and satisfy  $\mathbf{Q}^\mathit{T}\mathbf{Q}=\mathbf{I}_{K}$, and $\mathbf{R}$ is an upper triangular matrix. Then, we generate the ETF classifier $\mathbf{Z}$ using $\mathbf{Q}$:
\begin{equation}\label{con:eq_8} \small
	\mathbf{Z}=\sqrt{\frac{K}{K-1}} \mathbf{Q}\left(\mathbf{I}_{K}-\frac{1}{K} \mathbf{1}_{K} \mathbf{1}_{K}^{T}\right),
\end{equation}
where $\mathbf{Z}=\left[\mathbf{z}_{1}, \ldots, \mathbf{z}_{K}\right] \in \mathbb{R}^{d \times K}$ is the  concatenated matrix of $K$ column vectors and each vector $\mathbf{z}_{i} \in \mathbb{R}^{d \times 1}$ has the same dimension with the feature.

\textbf{Modification of the model architecture:}The ETF classifier is designed to align diverse features to predefined ETF vectors, facilitating knowledge fusion and enhancing separability between distinct knowledge. However, conventional neural networks (e.g., CNNs) typically produce sparse feature vectors (with many zero-valued dimensions) due to ReLU activations, whereas ETF vectors are inherently dense. To bridge this gap, a learnable projection layer $l(\omega_3,\cdot)$ is appended after the feature extractor \cite{li2023no}. For an input $x$, the process first extracts its feature $f(\omega_1, x )$, then projects it into the ETF-aligned space and normalizes output $\hat{\boldsymbol{h}}$ to obtain feature projection $\boldsymbol{h}$:
\begin{equation}\label{con:eq_9} \small
	\boldsymbol{h}=\hat{\boldsymbol{h}}/\|\hat{\boldsymbol{h}}\|_{2},\quad\hat{\boldsymbol{h}}=l(w_3, f(\omega_1, x)),
\end{equation}
where $w_1$ and $w_3$ are parameters of feature extractor and ETF projection layer.

\textbf{Dot regression loss for ETF classifier}: Neural collapse theory suggests that features of balanced data converge to class prototypes, which further collapse to their corresponding ETF vectors in the final training phase. To leverage this theory, we designs a dot regression loss for ETF classifier:
\begin{equation}\label{con:eq_10} \small
	\mathcal{L}_{DR}(\mathbf{h},\mathbf{Z})=\frac{1}{2}\left(\mathbf{h}_c^{T}\mathbf{z}_c-1\right)^2,
\end{equation}
where $c$ is the label of $\mathbf{h}$ and $\mathbf{Z}$ is the predefined ETF classifier. Since $\mathbf{h}$ and $\mathbf{z}_c$  are normalized vectors, the dot product between these two vectors is equal to 1 when  $\mathbf{h}_c$ collapses to $\mathbf{z}_c$. 

\textbf{Design of the dual-classifier structure:} In centralized learning, ETF has proven effective in replacing learnable classifiers to mitigate performance degradation under imbalanced data\cite{yang2022inducing}. However, simply using ETF projection $l(\omega_3,\cdot)$ and dot regression loss fails to ensure feature separability. As illustrated in Fig.\ref{fig:fig_2}, despite feature projections alignment with ETF vectors under $\mathcal{L}_{DR}$, discriminability between classes in low-dimensional space remains suboptimal. To address this, we integrate cross-entropy loss $\mathcal{L}_{CE}$ with the existing dot regression loss through dual-classifier collaboration, an ETF classifier $l(\omega_3,\cdot)$  and a trainable classifier $g(\omega_2,\cdot )$. The hybrid loss function is thus formulated as:
\begin{equation}\label{con:eq_11} \small 
	\mathcal{L}_{train}(\mathbf{h},\mathbf{Z},\mathbf{g},\mathbf{y}_c)=\lambda\mathcal{L}_{DR}(\mathbf{h},\mathbf{Z})+\mathcal{L}_{CE}(\mathbf{g},\mathbf{y}_c),
\end{equation}
where $\mathbf{g}=g(\omega_2,f(\omega_1,\mathbf{x}) )$, $\mathbf{y}_c$ is the label set of inputs $\mathbf{x}$ and $\lambda$ is the hyperparameter.

Compared with Fig.\ref{fig:fig_2a} and Fig.\ref{fig:fig_2b}, we can observe that $\mathcal{L}_{CE}$ enhances feature separability. In addition, experimental results demonstrate that models trained with objective Eq.(\ref{con:eq_11}) achieve superior performance (average accuracy on CIFAR10 = 0.732) over FedProto (0.725)\cite{tan2022fedproto}, despite eliminating class prototype exchange. This implies local optimization with Eq.(\ref{con:eq_11}) suffices to outperform federated prototype learning. Consequently, we propose a critical research question: \textbf{What constitutes essential information exchange in federated prototype learning?} 

\begin{figure}[htbp] 
	\centering
	\begin{subfigure}{0.48\linewidth}
		\centering
		\includegraphics[scale=0.22]{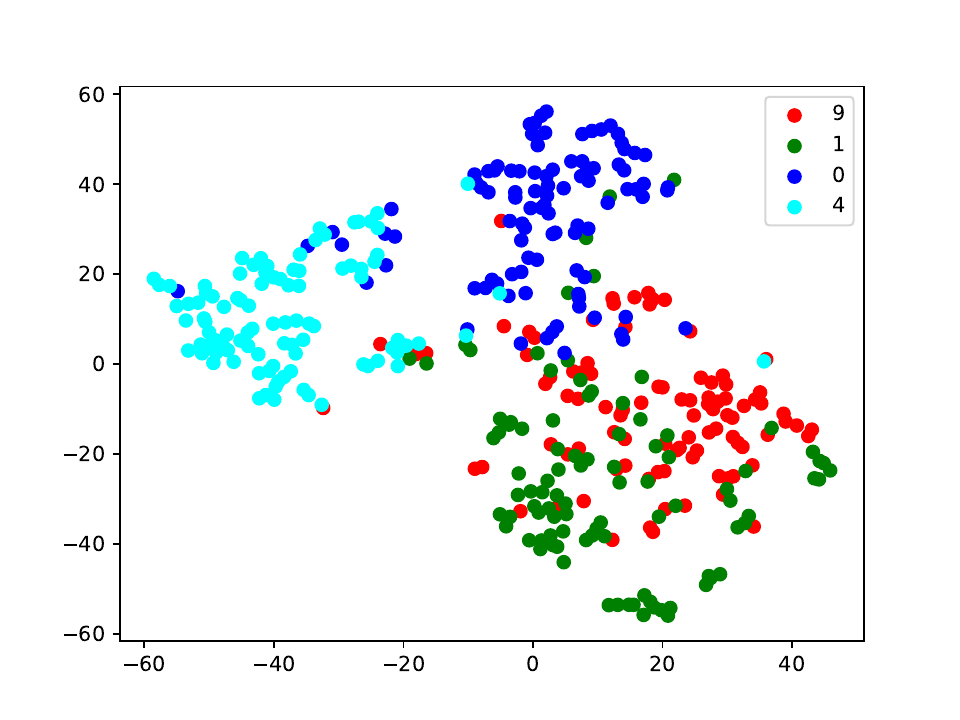}
		\caption{}
		\label{fig:fig_2a}
	\end{subfigure}
	\begin{subfigure}{0.48\linewidth}
		\centering
		\includegraphics[scale=0.22]{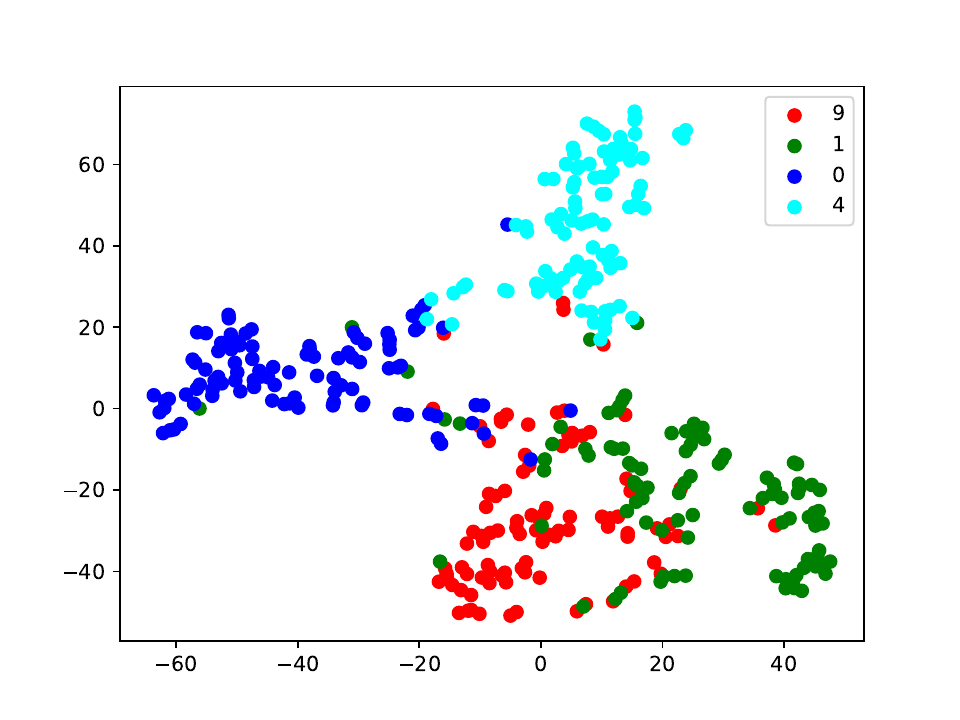}
		\caption{}
		\label{fig:fig_2b}
	\end{subfigure}
	
	\vspace{0.5cm} 
	
	\begin{subfigure}{0.48\linewidth}
		\centering
		\includegraphics[scale=0.22]{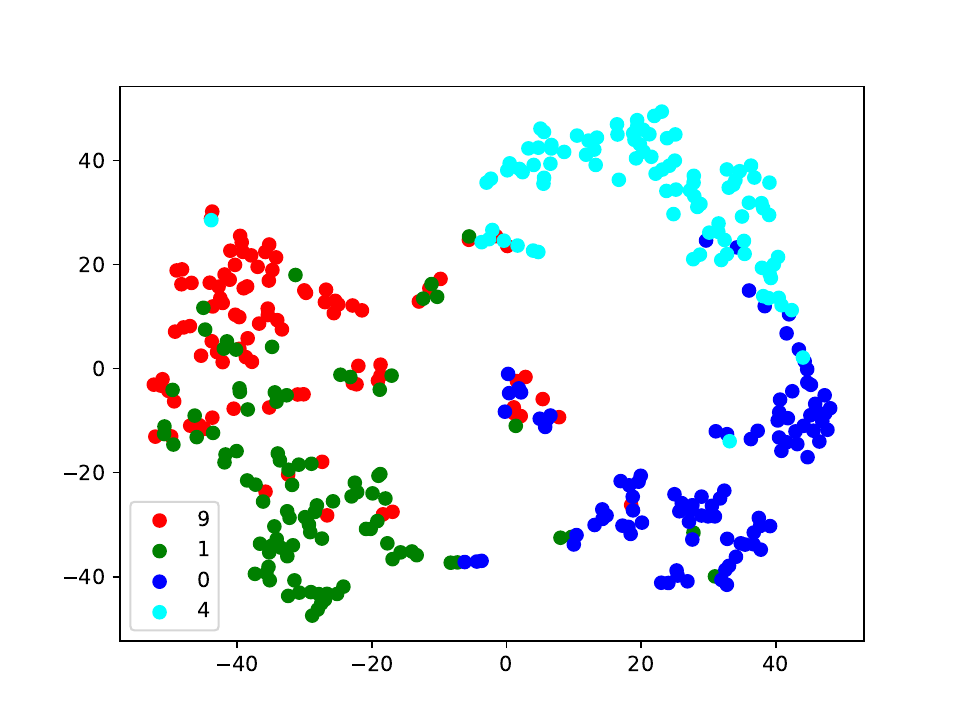}
		\caption{}
		\label{fig:fig_2c}
	\end{subfigure}
	\begin{subfigure}{0.48\linewidth}
		\centering
		\includegraphics[scale=0.22]{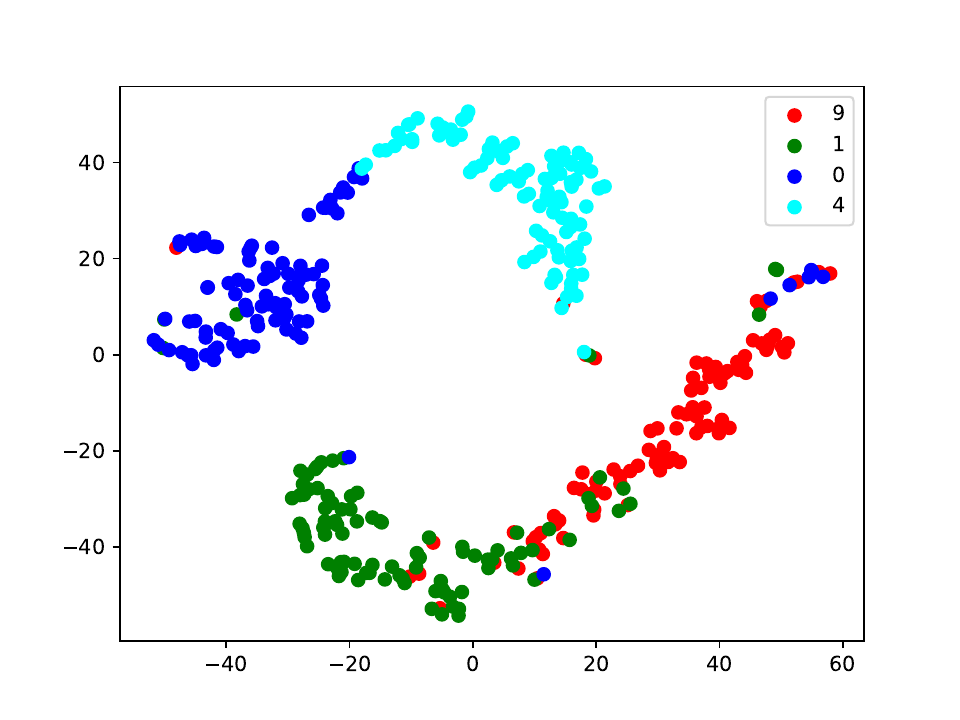}
		\caption{}
		\label{fig:fig_2d}
	\end{subfigure}
	
	\caption{Visualization of features and projections under different loss functions: (a) only $\mathcal{L}_{DR}$ (features); (b) $\mathcal{L}_{DR}+\mathcal{L}_{CE}$ (features); (c) only $\mathcal{L}_{DR}$ (projections); (d) $\mathcal{L}_{DR}+\mathcal{L}_{CE}$ (projections)}
	\label{fig:fig_2}
\end{figure}

Recall that FedProto \cite{tan2022fedproto} employs prototype alignment (local-to-global) to enhance local learning capacity, yet its generalization remains constrained by limited local data/feature. Although many studies \cite{luo2021no,li2024feature,chen2023federated} demonstrate that data generation methods can improve the generalization performance of local training (e.g., generating the similar data as client $m$ at client $1$), conventional data generation methods like GANs\cite{cao2022perfed} and diffusion models\cite{li2024feddiff} often suffer from fidelity-privacy trade-offs. To address this, GFPL introduces a prototype interaction and pseudo-feature generation  mechanism to enhance model generalization while preserving data privacy.

\subsection{Improving model generalization with Pseudo Feature Generation (PFG)}

The purpose of pseudo feature generation is to generate diverse and uniform features to retrain projection layers of DCS, enabling diversified features mapping to ETF vectors.

\textbf{Prototype generation and interaction}:Given input $x \in \mathbb{R} ^{C\times W\times H} $, the feature extractor generates $f(\omega_1, x )\in \mathbb{R}^{d}$. Let  $ \mathbb{C}=\left \{ \mathbb{C}_i \right \}_{i \in [m]}$ denote the total class set, where client $i$ has $ N_{i,j}$ samples for class $j$. We employ a GMM with $n$ components to fit the class feature distributions and form the local prototype $\textbf {\textit{P}}_{i,j}$ for class \textit{j} at the client \textit{i} with $(\boldsymbol{\mu}_{i,j},\boldsymbol{\sigma}_{i,j},\boldsymbol{\pi}_{i,j})$,
where $\boldsymbol{\mu}_{i,j}\in \mathbb{R}^{d\times n}$ is the mean matrix, $\boldsymbol{\sigma}_{i,j} \in \mathbb{R}^{d\times n}$ is the standard deviation, $\boldsymbol{\pi}_{i,j}$ contains the probability of each GMM component. Thus, $\textbf {\textit{P}}_{i,j}$ encodes statistical characteristics of class \textit{j} at the client \textit{i}. 

To achieve effective interaction and integration of knowledge, client \textit{i} uploads all local prototypes $\left \{ \textbf{\emph{P}} _{i,j}  \right \}_{j\in \mathbb{C}_i}  $ to the server. Having received all prototypes $\left \{ \textbf{\emph{P}} _{i,j}  \right \}_{i\in [m],j\in \mathbb{C}_i}$, the server forms the global prototype $\left \{ \textbf{\emph{P}} _{G,j}  \right \}$of class $j$ with Algorithm 1. Without loss of generality, we denote $\left \{ \textbf{\emph{P}} _{G,j}  \right \}$ as $\textbf{\emph{P}} _{G}$.
Specifically, the local prototypes of the same class will be fused based on the Bhattacharyya distance. Prototypes with a distance smaller than the threshold $S_C$ will undergo a weighted average, while those with a distance greater than $S_C$ will be retained, thereby generating a new global prototype for each class.

\textbf{Projection layer retraining with PFG}: After acquiring global prototypes $\textbf{\emph{P}} _{G}$, clients employ them to generate balanced pseudo features by replacing local GMM parameters ${(\boldsymbol{\mu}_{i},\boldsymbol{\sigma}_{i},\pi_{i})}_{i \in [m]}$ with $\textbf{\emph{P}} _{G}$. For each class $j \in \mathbb{C}$, we sample $r$ pseudo-features $\boldsymbol{f}_{j} = \left \{f_{j,1},f_{j,2},...,f_{j,r} \right \}$ from the updated GMM with Eq.(\ref{con:eq_4}), ensuring pseudo-features of each class balanced. Afterwards, these shuffled features are batched into $B$ and fed into DCS for retraining $g$ and $h$ with the retaining loss:
\begin{equation}\label{con:eq_13} \small
	\begin{aligned}
		&\mathcal{L}_{retrain}(\tilde{\mathbf{h}},\mathbf{Z},\tilde{\mathbf{g}},\mathbf{y}_c)=\lambda\mathcal{L}_{DR}(\tilde{\mathbf{h}},\mathbf{Z})+\mathcal{L}_{CE}(\tilde{\mathbf{g}},\mathbf{y}_c),\\
		&\tilde{\boldsymbol{h}}=\hat{\boldsymbol{h}}/\|\hat{\boldsymbol{h}}\|_{2},\quad\hat{\boldsymbol{h}}={h}(w_3,\boldsymbol{f}_j), \quad \tilde{\mathbf{g}} = {g}(w_2,\boldsymbol{f}_j).
	\end{aligned}
\end{equation}

GFPL strategically defers PFG and projection retraining until round $t_1$, allowing feature extractors to remain frozen. Although this delayed PFG postpones prototype interaction, it effectively reduces communication overhead while maintaining learning efficiency. Furthermore, to further minimize communication costs, a projection retraining interval parameter $S_T$ is introduced. Specifically, prototype interaction and projection retraining are performed once every $S_T$ rounds, thereby reducing the number of communication rounds from the original $T$ to $\lfloor T / S_T \rfloor$, which significantly cuts down the overall system communication overhead. The detailed procedure of GFPL is summarized in Algorithm~2. For the convergence proof of GFPL, see Appendix A.

\begin{algorithm}[H]
	\small
	\caption{Prototype Fusion using Bhattacharyya Distance}
	\label{alg:prototype_fusion}
	\begin{algorithmic}[1]
		\Require Local prototypes $\mathcal{U} \gets \{\mathbf{P}_{i,j}\}_{j \in \mathbb{C}}$, threshold $S_C$
		\State Initialize $\mathbf{P}_G \gets \emptyset$, cluster set $\mathcal{C} \gets \emptyset$
		\While{$\mathcal{U} \neq \emptyset$}
		\State Select $p_{\text{seed}}$ from $\mathcal{U}$, $\mathcal{C}_{\text{current}} \gets \{p_{\text{seed}}\}$, $\mathcal{U} \gets \mathcal{U} \setminus \{p_{\text{seed}}\}$
		\Repeat
		\State $\mathcal{N} \gets \left\{ p \in \mathcal{U} \mid D_B(p, p') < S_C,\ \forall p' \in \mathcal{C}_{\text{current}} \right\}$
		\If{$\mathcal{N} = \emptyset$} \textbf{break} \EndIf
		\State $\mathcal{C}_{\text{current}} \gets \mathcal{C}_{\text{current}} \cup \mathcal{N}$, $\mathcal{U} \gets \mathcal{U} \setminus \mathcal{N}$
		\Until{no new prototypes added}
		\State $\mathcal{C} \gets \mathcal{C} \cup \{\mathcal{C}_{\text{current}}\}$
		\EndWhile
		\For{each cluster $\mathcal{C}_i \in \mathcal{C}$}
		\State $w_{\text{fused}} = \sum_{j \in \mathcal{C}_i} w_j$
		\State $\boldsymbol{\mu}_{\text{fused}} = \frac{1}{w_{\text{fused}}} \sum_{j \in \mathcal{C}_i} w_j \boldsymbol{\mu}_j$
		\State $\boldsymbol{\Sigma}_{\text{fused}} = \frac{1}{w_{\text{fused}}} \sum_{j \in \mathcal{C}_i} w_j \left[ \boldsymbol{\Sigma}_j + (\boldsymbol{\mu}_j - \boldsymbol{\mu}_{\text{fused}})(\boldsymbol{\mu}_j - \boldsymbol{\mu}_{\text{fused}})^T \right]$
		\State $\mathbf{P}_G \gets \mathbf{P}_G \cup \{(\boldsymbol{\mu}_{\text{fused}}, \boldsymbol{\Sigma}_{\text{fused}}, w_{\text{fused}})\}$
		\EndFor
		\State \Return $\mathbf{P}_G$
	\end{algorithmic}
\end{algorithm}

\begin{algorithm}[H]
	\small
	\caption{GFPL}
	\label{alg:gfpl}
	\begin{algorithmic}[1]
		\Require Client datasets $\left\{ \mathbb{D}_i \right\}_{i=1}^m$, weights $\left\{w_i\right\}_{i=1}^m$, $T,S_T,t_1$
		\Statex \textbf{Server executes:}
		\State Initialize global ETF $\mathbf{Z}$ via Eq.(\ref{con:eq_8}), global prototype $\mathbf{P}_G$
		\For{each round $t = t_1, \cdots, T$}
		\For{each client $i$ \textbf{in parallel}}
		\State $\{\mathbf{P}_{i,j}\}_{j\in\mathbb{C}} \gets \text{LocalUpdate}(i, \mathbf{Z}, \mathbf{P}_G)$
		\EndFor
		\State $\mathbf{P}_G \gets \text{PrototypeFusion}(\{\mathbf{P}_{i,j}\}_{j\in\mathbb{C}})$
		\State Update local prototypes $\{\mathbf{P}_{i,j}\}_{j\in\mathbb{C}}$ with $\mathbf{P}_G$
		\EndFor
		
		\Statex \textbf{LocalUpdate:}
		\For{each local epoch $E$}
		\For{each batch $(x_j,y_j) \in \mathbb{D}_i$}
		\State Compute $\mathcal{L}_{train}$ via Eq.(\ref{con:eq_11})
		\State Update extractor and classifier using $\mathcal{L}_{train}$
		\If{$t \ge t_1$ and $t \bmod S_T = 0$}
		\State Generate pseudo features via Eq.(\ref{con:eq_4})
		\State Compute $\mathcal{L}_{retrain}$ via Eq.(\ref{con:eq_13})
		\State Update projection layer using $\mathcal{L}_{retrain}$
		\EndIf
		\State Generate local prototypes $\{\mathbf{P}_{i,j}\}_{j\in\mathbb{C}}$ via Eq.(\ref{con:eq_5}) and (\ref{con:eq_6})
		\EndFor
		\EndFor
		\State \Return $\{\mathbf{P}_{i,j}\}_{j\in\mathbb{C}}$
	\end{algorithmic}
\end{algorithm}

\textbf{Privacy preserving:} Instead of exchanging model parameter, GFPL only uploads local prototypes to the server, generated with low-dimension representation of samples.In Appendix B, we provide rigorous proofs for information-theoretic impossibility of feature reconstruction and optimization-theoretic impossibility of data reconstruction. Thus, the training process is reliability for model invisibility and prototype-to-data irreversibility.

\section{Experiments}
\label{sec:EXP}
\subsection{Experiments setup}

\begin{table*}[htbp]\label{tabel_1}
	\centering
	\caption{Comparison of Federated Learning Algorithms}\scriptsize
	\captionsetup{font={tiny}}
	\begin{tabular}{cccccccc}
		\toprule[0.5mm]
		\multirow{2}[4]{*}{Dataset} & \multirow{2}[4]{*}{Method} & \multirow{2}[4]{*}{$(\hat{w},\hat{s})$} & \multicolumn{3}{c}{Average Test Accuracy(\%)} & \multicolumn{1}{c}{\multirow{2}[4]{*}{\shortstack{Comm.\\Round}}} & \multicolumn{1}{c}{\multirow{2}[4]{*}{\shortstack{Comm. Param.\\ $(\times 10^3)$}}} \\
		\cmidrule{4-6}          &       &       & $\bar{w}=3 $  & $\bar{w}=4$   & $\bar{w}=5$   &       &  \\
		\midrule[0.5mm]
		\multirow{8}[2]{*}{MNIST} & Local & (2,2)     & 94.05$\pm$2.93 & 93.35$\pm$3.26 & 92.92$\pm$3.17 & 0   & 0 \\
		& FeSEM (2020) & (2,2)      & 95.26$\pm$3.48 & 97.06$\pm$2.72 & 96.31$\pm$2.41 & 150   & 430 \\
		& FedProx (2020) & (2,2)      & 96.26$\pm$2.89 & 96.40$\pm$3.33 & 95.65$\pm$3.38 & 110   & 430 \\
		& FedPer (2019) & (2,2)      & 95.57$\pm$2.96 & 96.44$\pm$2.62 & 95.55$\pm$3.13 & 100   & 106 \\
		& FedAvg (2017) & (2,2)      & 95.04$\pm$6.48 & 94.32$\pm$4.89 & 93.22$\pm$4.39 & 150   & 430 \\
		& FedRep (2021) & (2,2)      & 94.96$\pm$2.78 & 95.18$\pm$3.80 & 94.94$\pm$2.81 & 100   & 110 \\
		& FedProto (2022) & (2,2)      & 97.13$\pm$0.30 & 96.80$\pm$0.41 & 96.70$\pm$0.29 & 100   & 4 \\
		& FedFA (2023) &(2,2)       &95.33$\pm$1.24 & 95.67 $\pm$1.88 & 95.32$\pm$1.98 & 110     & 430 \\
		& FedPA (2024) &(2,2)       &96.36$\pm$0.33 & 95.48 $\pm$0.59 & 94.82$\pm$0.43 & 120     & 504 \\
		& \textbf{GFPL(ours)} & (2,2)      & \textbf{98.91$\pm$0.22} & \textbf{98.01$\pm$0.32} & \textbf{97.83$\pm$0.21} & 80    & 2 \\
		\midrule
		\multirow{8}[2]{*}{FENNIST} & Local & (1,1)     & 92.50$\pm$10.42 & 91.16$\pm$5.64 & 87.91$\pm$8.44 & 0     & 0 \\
		& FeSEM (2020) & (1,1)     & 93.39$\pm$6.75 & 91.06$\pm$6.43 & 89.61$\pm$7.89 & 200   & 16,000 \\
		& FedProx (2020) & (1,1)     & 94.53$\pm$5.33 & 90.71$\pm$6.24 & 91.33$\pm$7.32 & 300   & 16,000 \\
		& FedPer (2019) & (1,1)     & 93.47$\pm$5.44 & 90.22$\pm$7.63 & 87.73$\pm$9.64 & 250   & 102 \\
		& FedAvg (2017) & (1,1)    & 94.50$\pm$5.29 & 91.39$\pm$5.23 & 90.95$\pm$7.22 & 300   & 16,000 \\
		& FedRep (2021) & (1,1)     & 93.36$\pm$5.34 & 91.41$\pm$5.89 & 89.98$\pm$6.88 & 200   & 102 \\
		& FedProto (2022) & (1,1)     & 96.82$\pm$1.75 & 94.93$\pm$1.61 & 93.67$\pm$2.23 & 120   & 4 \\
		& FedFA (2023) &(1,1)      &94.16$\pm$2.26 & 92.67 $\pm$3.85 & 91.32$\pm$4.77 & 160    & 16,000 \\
		& FedPA (2024) &(1,1)     &95.27$\pm$1.99 & 94.42 $\pm$2.35 & 92.76$\pm$2.87 & 180    & 18,000 \\
		& \textbf{GFPL(ours)} &(1,1)     & \textbf{97.56$\pm$1.65} & \textbf{96.48$\pm$1.52} & \textbf{94.78$\pm$1.77} & 110    & 2 \\
		\midrule
		\multirow{8}[1]{*}{CIAFR10} & Local & (1,1)     & 67.72$\pm$9.45 & 55.62$\pm$7.15 & 51.64$\pm$6.57 & 0     & 0 \\
		& FeSEM (2020) & (1,1)   & 68.19$\pm$3.31 & 64.40$\pm$3.23 & 62.17$\pm$3.51 & 120   & 235,000 \\
		& FedProx (2020) & (1,1)    & 71.25$\pm$2.44 & 67.20$\pm$1.31 & 64.19$\pm$2.23 & 150   & 235,000 \\
		& FedPer (2019) & (1,1)     & 72.38$\pm$4.58 & 66.73$\pm$4.59 & 64.21$\pm$4.27 & 130   & 235,000 \\
		& FedAvg (2017) & (1,1)    & 69.72$\pm$2.77 & 64.77$\pm$2.37 & 63.74$\pm$2.61 & 150   & 235,000 \\
		& FedRep (2021) & (1,1)     & 69.44$\pm$10.48 & 64.93$\pm$7.46 & 61.36$\pm$7.04 & 110   & 235,000 \\
		& FedProto (2022) & (1,1)    & 72.49$\pm$1.97 & 67.12$\pm$2.03 & 65.08$\pm$1.98 & 110   & 41 \\
		& FedFA (2023) &(1,1)      &71.11$\pm$2.78 & 66.23 $\pm$3.45 & 63.25$\pm$4.98 & 140    & 235,000 \\
		& FedPA (2024) &(1,1)      &70.24$\pm$2.33 & 65.98 $\pm$2.94 & 63.48$\pm$2.87 & 160    & 236,000 \\
		& \textbf{GFPL(ours)} & (1,1)     & \textbf{74.23$\pm$1.60} & \textbf{70.22$\pm$1.78} & \textbf{68.72$\pm$1.33} & 100    & 33 \\
		\midrule
		\multirow{8}[1]{*}{CIAFR100} & Local & (2,2)     & 56.22$\pm$9.77 & 52.51$\pm$8.75 & 50.33$\pm$10.77 & 0     & 0 \\
		& FeSEM (2020) & (2,2)   & 62.15$\pm$4.22 & 61.52$\pm$4.44 & 63.26$\pm$3.21 & 120   & 235,000 \\
		& FedProx (2020) & (2,2)    & 61.23$\pm$3.24 & 64.11$\pm$2.33 & 64.33$\pm$2.33 & 130   & 235,000 \\
		& FedPer (2019) & (2,2)     & 67.34$\pm$3.58 & 62.35$\pm$3.59 & 62.54$\pm$4.55 & 120   & 235,000 \\
		& FedAvg (2017) & (2,2)    & 66.23$\pm$4.27 & 61.48$\pm$3.35 & 62.76$\pm$2.44 & 140   & 235,000 \\
		& FedRep (2021) & (2,2)     & 68.27$\pm$7.43 & 62.64$\pm$7.68 & 63.56$\pm$7.23 & 120   & 235,000 \\
		& FedProto (2022) & (2,2)    & 67.76$\pm$2.88 & 66.56$\pm$2.33 & 64.88$\pm$1.56 & 110   & 41 \\
		& FedFA (2023) &(2,2)      &67.01$\pm$2.48 & 65.11 $\pm$3.55 & 64.21$\pm$4.37 & 140    & 235,000 \\
		& FedPA (2024) &(2,2)      &65.24$\pm$2.52 & 62.78 $\pm$2.04 & 63.43$\pm$2.26 & 150    & 236,000 \\
		& \textbf{GFPL(ours)} & (2,2)     & \textbf{70.11$\pm$1.10} & \textbf{68.32$\pm$1.23} & \textbf{67.23$\pm$1.55} & 80    & 20 \\		
		\bottomrule[0.5mm]
	\end{tabular}%
	\label{tab:table1}%
\end{table*}

\begin{table*}[htbp]
	\centering 
	\caption{The effectiveness of GFPL's components in performance improvements of CIFAR10}\scriptsize
	\begin{tabular}{ccccccc}
		\toprule[0.5mm]
		Variant & CIFAR10($\hat{s}=1$) & CIFAR10($\hat{s}=20$) &  FEMNIST($\hat{s}=1$)  & FEMNIST($\hat{s}=20$) &  CIFAR100($\hat{s}=1$)  & CIFAR100($\hat{s}=20$) \\
		\midrule[0.2mm]
		GFPL w/o DCS and PFG  & 55.62$\pm$7.15 &  54.23$\pm$9.88    &   91.16$\pm$5.64    & 90.22$\pm$7.84  &   53.51$\pm$8.75    & 50.31$\pm$9.84 \\
		GFPL w/o DCS & 64.47$\pm$3.36 & 63.28$\pm$5.77 &   93.25$\pm$4.22    & 92.33$\pm$6.22  &   62.34$\pm$6.33    & 60.63$\pm$7.23\\
		GFPL w/o PFG  &  67.24$\pm$2.88 & 66.01$\pm$3.95 & 94.37$\pm$3.43 & 92.74$\pm$3.88 & 67.01$\pm$4.31    & 66.21$\pm$5.44\\
		\midrule[0.2mm]
		GFPL(Ours) & \textbf{70.22$\pm$1.78} & \textbf{68.64$\pm$2.55} & \textbf{96.48$\pm$1.52} & \textbf{94.82$\pm$2.85} &   \textbf{70.33$\pm$1.12}    & \textbf{68.26$\pm$2.02} \\
		\bottomrule[0.5mm]
	\end{tabular}%
	\label{tab:table2}%
\end{table*}
\textbf{Dataset and model architecture:} We assume that there are $m=20$ clients and $1$ server involved in FL and consider image classification task adopting three popular benchmark datasets: MNIST \cite{deng2012mnist} (10 classes), FEMNIST \cite{caldas2018leaf} (62 classes), CIFAR10 \cite{krizhevsky2009learning} (10 classes) and CIFAR100 \cite{krizhevsky2009learning} (100 classes). A multi-layer CNN which consists of 2 convolutional layers then 2 fully connected layers is for both MNIST and FEMNIST, and ResNet18 \cite{he2016deep} is for CIFAR10 and CIFAR100. To meet the generation conditions of ETF ($d>K$), we set $d=50,80,128,128$ respectively for MNIST, FEMNIST, CIFAR10, and CIFAR100 which corresponds to the output dimension of the feature extractor.

\textbf{Imbalanced dataset scenarios:} Consistent with the FedProto setting\cite{tan2022fedproto}, we borrow the concept of $w$-way $s$-shot from few-shot learning where $w$ controls the number of classes and $s$ controls the number of instances per class. To simulate imbalanced data scenarios, we define the average values $\bar{w}$ and $\bar{s}$, and the variances $\hat{w}$ and $\hat{s}$, where $\hat{w}$ and $\hat{s}$ control the degree of label skew and feature skew.

\textbf{Baselines:} We compare the performance of GFPL with baselines under imbalanced data setting, including local training with CE loss,  FedAvg\cite{mcmahan2017communication}, FedProx\cite{li2020federated}, FeSEM \cite{long2023multi}, FedPer \cite{arivazhagan2019federated}, FedRep \cite{collins2021exploiting}, FedProto\cite{tan2022fedproto},  FedFA\cite{zhou2023fedfa} and FedPA\cite{jiang2024fedpa}.

\textbf{Hyperparameters:} All experiments are conducted on 1 workstation configured with RTX 3090-24G GPU, 10-core Intel(R) i9-10900K 3.70GHz CPU, 5TB SSD, 64GB RAM. We implement GFPL and the baselines in PyTorch, where hyperparameter $\lambda=2$ unless otherwise stated. The number of components in each GMM  $n=4$ and the pseudo-feature number of each class $r=16$. Projection retraining interval parameter $S_T=10$ and $t_1=10$. A detailed setup of other hyperparameters is given in Appendix C.

\begin{figure*}[h] 
	\centering
	\begin{subfigure}{0.28\linewidth}
		\centering
		\includegraphics[scale=0.2]{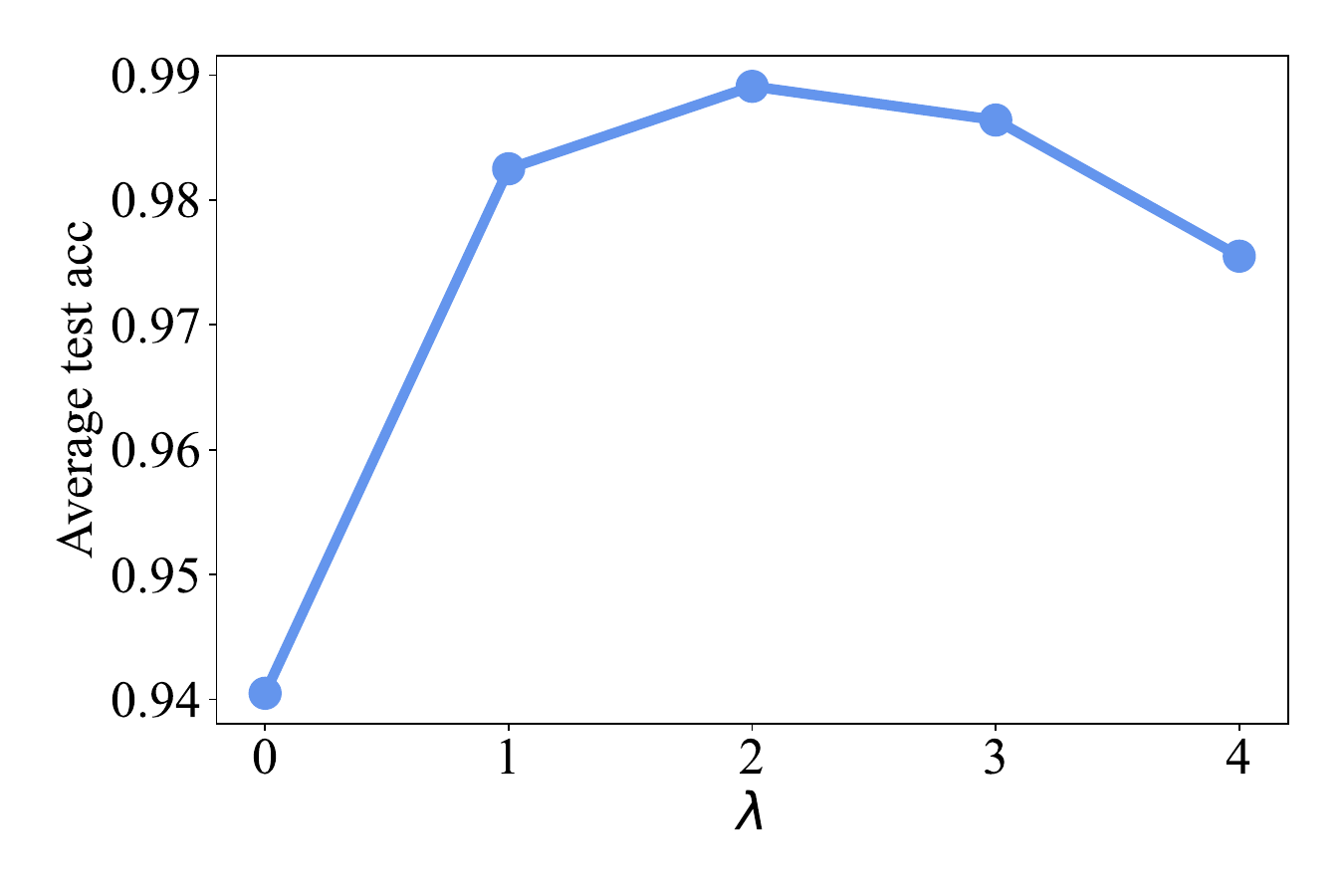}
		\caption{}
		\label{fig:fig_3a}
	\end{subfigure}
	\begin{subfigure}{0.28\linewidth}
		\centering
		\includegraphics[scale=0.2]{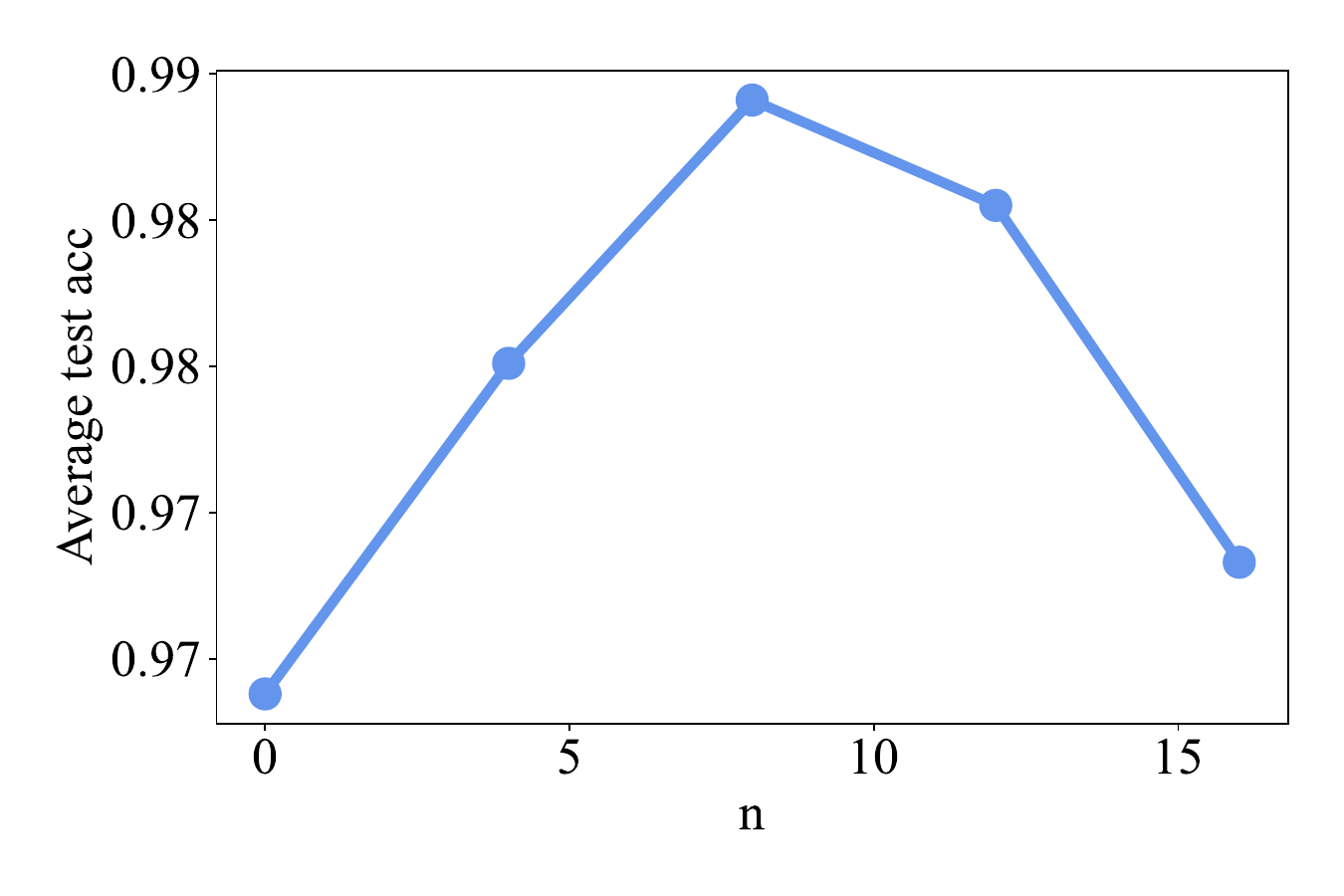}
		\caption{}
		\label{fig:fig_3b}
	\end{subfigure}
	\begin{subfigure}{0.28\linewidth}
		\centering
		\includegraphics[scale=0.2]{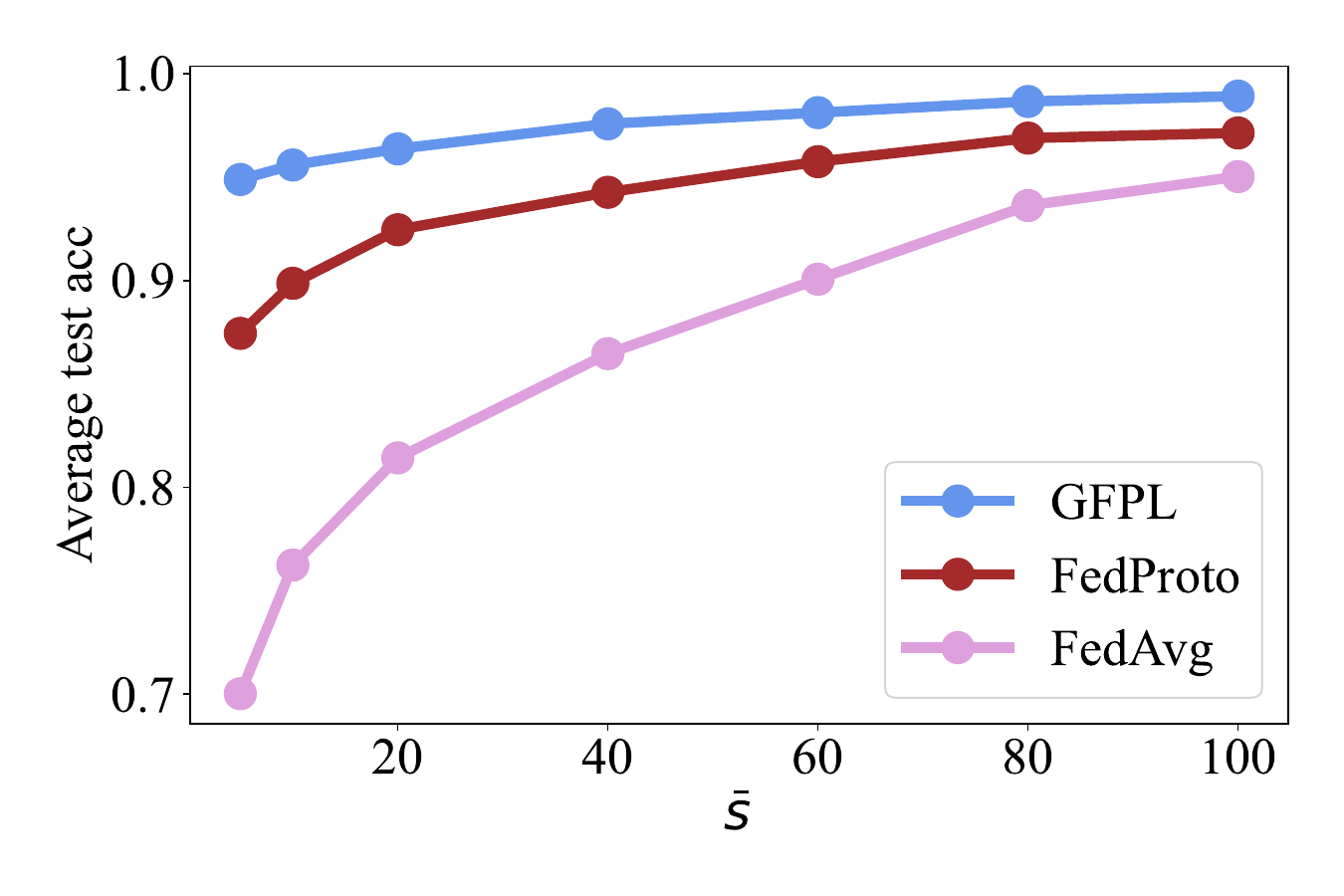}
		\caption{}
		\label{fig:fig_3c}
	\end{subfigure}
	\caption{The influence of hyperparameters on model performance: (a) $\lambda$; (b) GMM commponent number; (c)The average of shot.}
	\label{fig:hyperparameter}
\end{figure*}
\begin{figure*}[htbp] 
	\centering
	\begin{subfigure}{0.28\linewidth}
		\centering
		\includegraphics[scale=0.2]{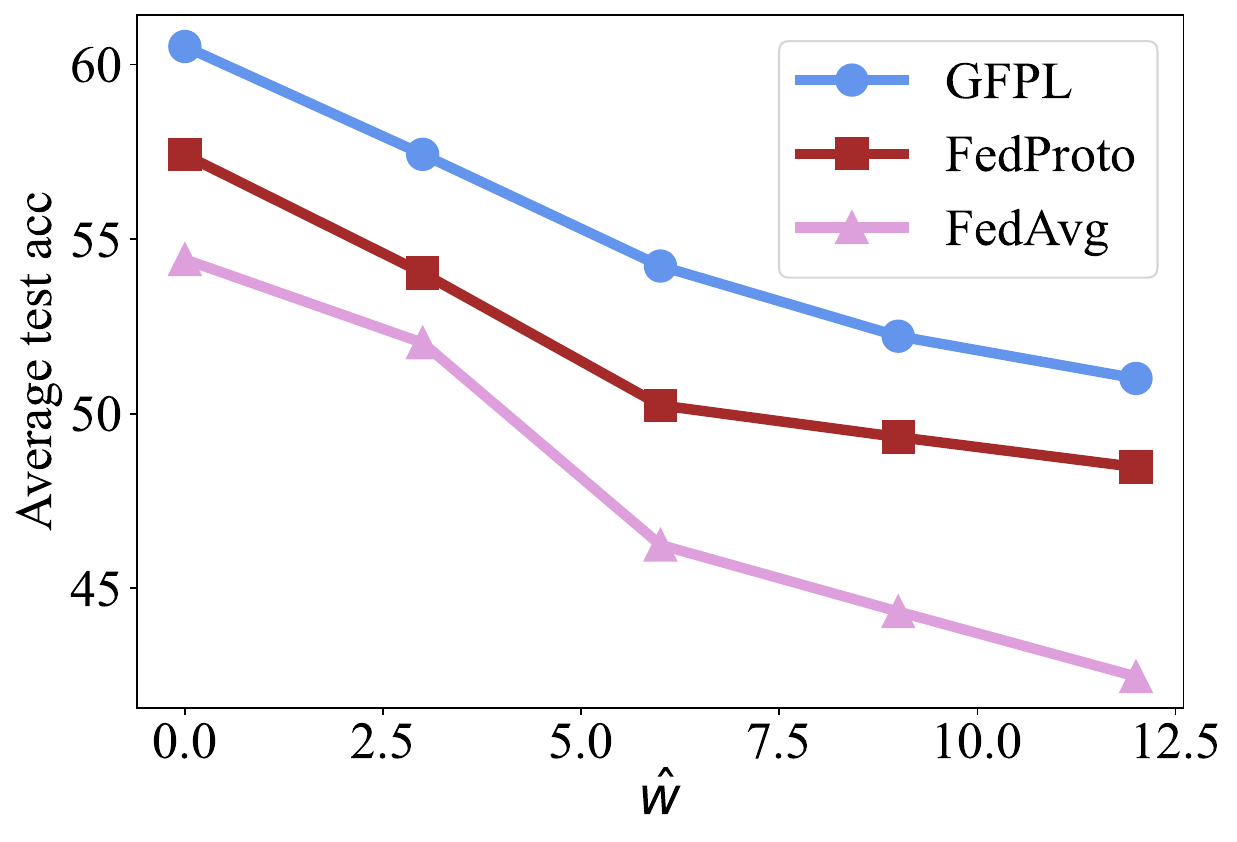}
		\caption{}
		\label{fig:fig_4a}
	\end{subfigure}
	\begin{subfigure}{0.28\linewidth}
		\centering
		\includegraphics[scale=0.2]{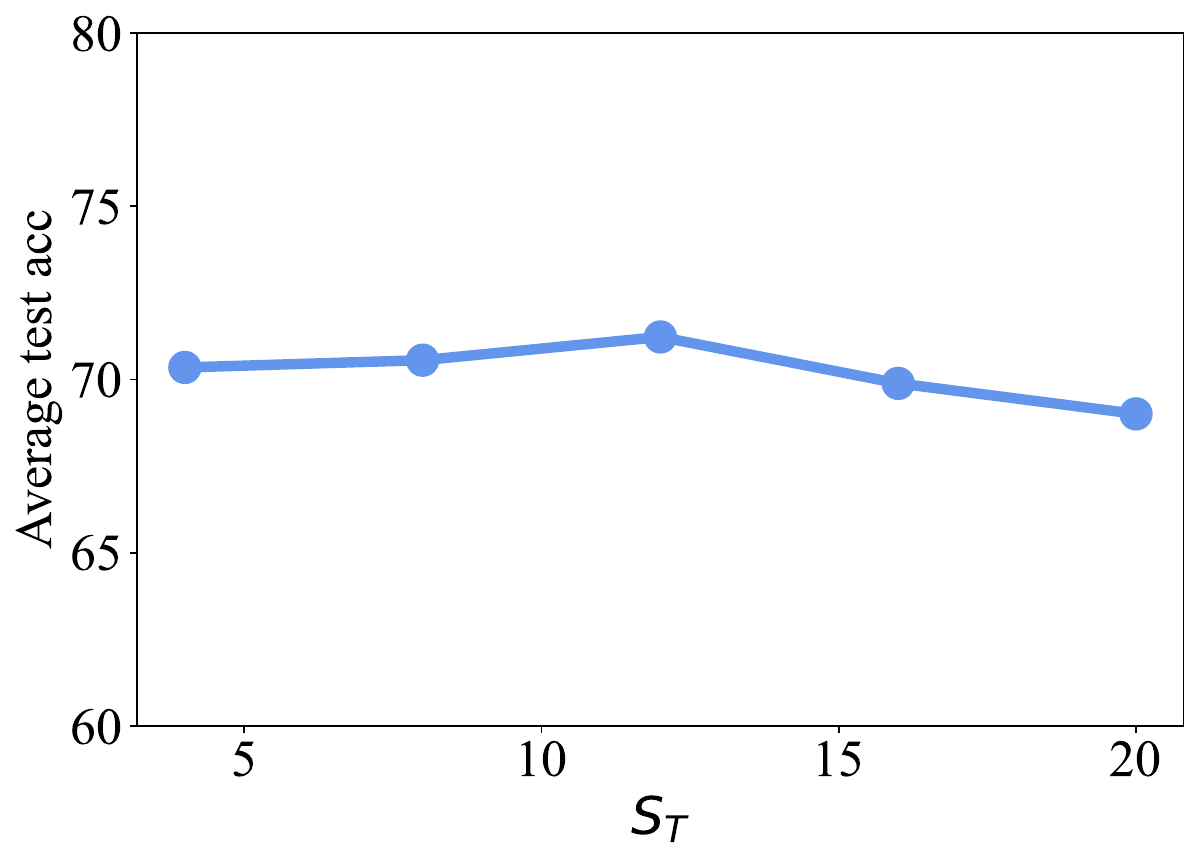}
		\caption{}
		\label{fig:fig_4b}
	\end{subfigure}
	\begin{subfigure}{0.28\linewidth}
		\centering
		\includegraphics[scale=0.2]{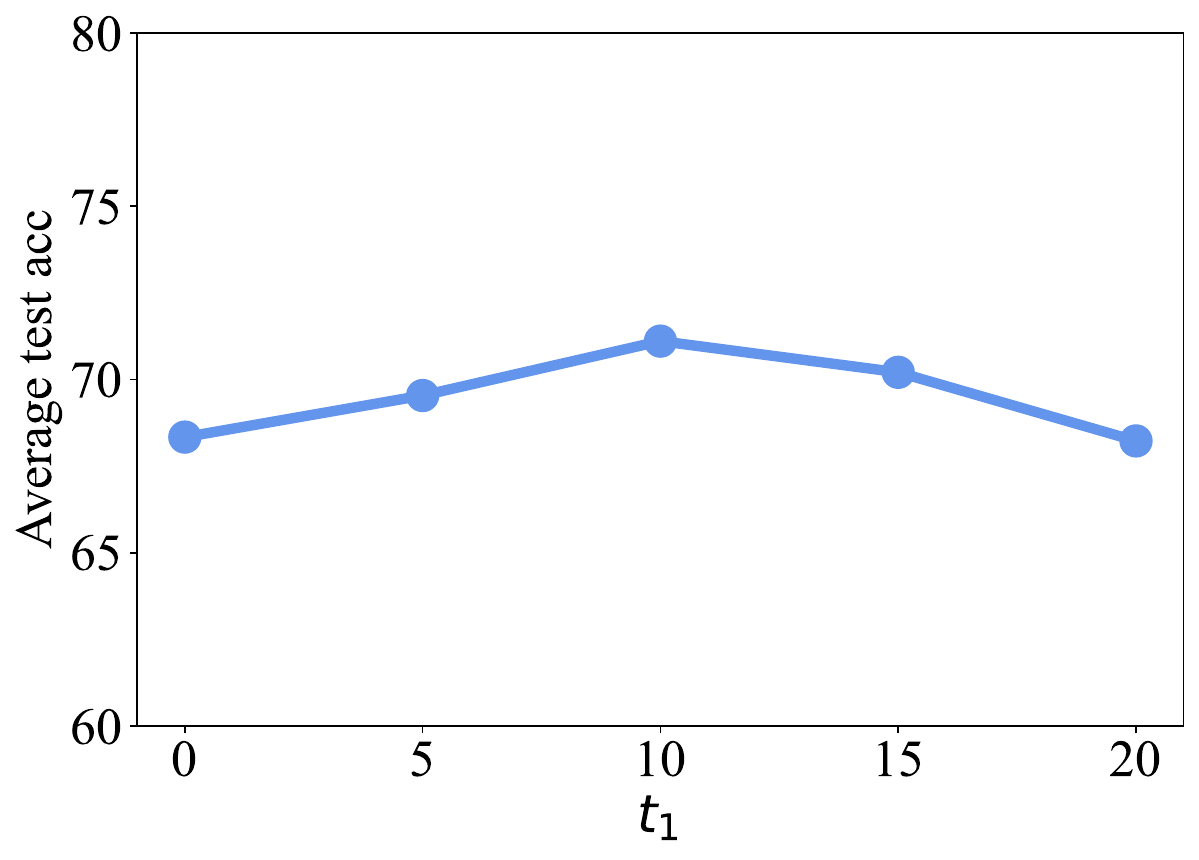}
		\caption{}
		\label{fig:fig_4c}
	\end{subfigure}
	\caption{The influence of hyperparameters on model performance: (a)$\hat{w}$; (b)Retraining interval; (c)Initial round of prototype interaction}
	\label{fig:hyperparameter2}
\end{figure*}

\subsection{Performance analysis under imbalanced data}

\textbf{Comparative study:} We conduct experiments on GFPL and baselines under imbalanced data, where $\bar{s}$ is fixed. The average test accuracy over all clients is shown in Table \ref{tab:table1}. We can conclude that GFPL not only obtains the highest accuracy on benchmark datasets, but also reaches convergence the fastest (besides local training). Obviously, GFPL improves the highest average test accuracy by 3.6\% on CIFAR10 dataset. The error bars are presented in Appendix C. 

\textbf{Ablation study:} To validate the effectiveness of each component in GFPL, we conduct ablation studies on three datasets, with results summarized in Table \ref{tab:table2} ($\bar{w}=4$, $\hat{w}=1$, $\bar{s}=100$). Compared to the baseline "GFPL w/o DCS and PFG" – local training with a single classifier using only $\mathcal{L}_{CE}$ – both "GFPL w/o PFG" (dual-classifier training with $\mathcal{L}_{train}$ in Eq.(\ref{con:eq_11})) and "GFPL w/o DCS" (single-classifier training with pseudo-feature generation) achieve significant improvements across all settings. Furthermore, by integrating both DCS and PFG, GFPL attains the highest test accuracy in every setting, underscoring the essential role of these two components in the proposed framework.

\textbf{Performance influence of hyperparameters:} Fig.\ref{fig:hyperparameter} analyzes the impact of hyperparameters $\lambda$, $n$, and $\bar{s}$ on MNIST. As Fig.\ref{fig:fig_3a} shows, accuracy rises sharply from $0.9405$ to $0.9891$ as $\lambda$ increases from $0$ to $2$, confirming the effectiveness of the $\mathcal{L}_{DR}$ loss in Eq.($\ref{con:eq_10}$). Fig.\ref{fig:fig_3b} indicates that the component count $n$ peaks at $8$. This phenomenon is attributed to the fact that insufficient components (small $n$) fail to capture complex feature distributions, leading to underfitting, while excessive components (large $n$) may overfit noise, thereby weakening generalization ability. In Fig.\ref{fig:fig_3c}, GFPL maintains more stable accuracy than FedProto and FedAvg as client sample size $\hat{s}$ decreases, indicating its superior scalability across different sample sizes.

Additionally, we conducted performance analysis experiments on $\hat{w}$, projection layer retraining interval $S_T$, and the initial round for prototype interaction $t_1$, with results presented in Fig.\ref{fig:fig_4a}-\ref{fig:fig_4c}. The findings demonstrate that $\text{GFPL}$ consistently outperforms $\text{Fedproto}$ across different $\hat{w}$ settings, which is mainly attributed to its well-designed feature alignment method and pseudo-feature generation mechanism. Regarding $S_T$, the results suggest that a smaller value is not always optimal, as overly frequent information exchange (small $S_T$) result in insufficient local training within a single communication round. Conversely, a moderately increased $S_T$ effectively raises the number of local training epochs, thereby reducing the total number of required communication rounds. Finally, we observed that the selection of $t_1$ requires careful consideration. It is generally preferable to select a moment when the local features have achieved a certain degree of intra-class cohesion. Initiating prototype interaction too early may lead to unrepresentative initial prototypes, whereas excessively delaying it can cause untimely information fusion, resulting in local training converging to suboptimal solutions and thus hindering the rapid improvement of model accuracy. Further analysis of hyperparameters is provided in Appendix C.

\textbf{Communication overhead:} Table $\ref{tab:table1}$ demonstrates that $\text{GFPL}$ achieves the highest accuracy improvement with minimal communication overhead. This superior efficiency is primarily attributed to the designed delayed prototype interaction strategy and periodic prototype fusion mechanism. Here, all communication overhead originates solely from the interaction of prototypes across all clients, which aims to generate effective pseudo-features locally for retraining the projection layer. The experimental results underscore that selecting the appropriate timing for information interaction is a crucial research problem in federated learning, as it directly influences the trade-off between model accuracy and communication costs. Further analysis is provided in Appendix D.

\section{Conclusion}
\label{sec:Con}

In this paper, we innovatively integrate prototype learning and generative methods into the FL framework, thereby offering a novel perspective to enhance knowledge fusion and model generalization. Specifically, unlike previous approaches that directly utilize feature centroids, GFPL pioneers the use of GMM to derive feature prototypes and successfully implements prototype fusion based on the Bhattacharyya distance, thus opening new avenues for effective prototype learning and cross-client knowledge fusion. Looking ahead, future research could explore the potential of using more advanced generative models, such as GANs or diffusion models, as alternative tools for prototype extraction and pseudo-feature generation. However, their substantial parameter requirements currently render them unsuitable for resource-constrained federated learning scenarios. Further discussion and analysis regarding these directions are provided in Appendix D.

\clearpage
\setcounter{page}{1}
\appendix
\section{Convergence Analysis}
\label{sec:convergence_analysis}

This section provides a theoretical convergence guarantee for the proposed GFPL framework. We first formulate the optimization problem and then present the main convergence theorem along with its proof.

\subsection{Problem Formulation}

Let $\mathcal{D}_i$ denote the local dataset on client $i$. The overall goal is to minimize the global objective function:

\begin{equation}
	\label{eq:global_objective}
	F(\Theta) = \sum_{i=1}^N p_i \mathbb{E}_{(x, y) \sim \mathcal{D}_i} \left[ \mathcal{L}_{\text{total}}^i(x, y; \Theta) \right],
\end{equation}

where $\Theta = \{\theta_f, \theta_{h_1}, \theta_{h_2}\}$ represents the parameters of the feature extractor, the conventional classifier, and the projection layer, respectively. $p_i \geq 0$ is the weight for client $i$ with $\sum_i p_i = 1$. The local loss $\mathcal{L}_{\text{total}}^i$ for client $i$ is defined as:

\begin{equation}
	\label{eq:local_loss}
	\begin{split}
		\mathcal{L}_{\text{total}}^i = \mathcal{L}_{\text{CE}}\left( h_1(f(x; \theta_f); \theta_{h_1}), y \right) \\
		+ \lambda \mathcal{L}_{\text{align}}\left( h_2(f(x; \theta_f); \theta_{h_2}), \mathbf{v}_y \right),
	\end{split}
\end{equation}

where $\mathcal{L}_{\text{CE}}$ is the cross-entropy loss, $\mathcal{L}_{\text{align}}$ is the alignment loss (e.g., dot product loss) between the projected features and the ETF vector $\mathbf{v}_y$ for class $y$, and $\lambda$ is a hyperparameter balancing the two terms.

\subsection{Assumptions}

To facilitate the analysis, we make the following standard assumptions:

\begin{enumerate}
	\item \textbf{($L$-Smoothness)} The global objective function $F(\Theta)$ is $L$-smooth:
	\begin{equation}
		\| \nabla F(\Theta) - \nabla F(\Theta') \| \leq L \| \Theta - \Theta' \|, \quad \forall \Theta, \Theta'.
	\end{equation}
	
	\item \textbf{(Bounded Gradient)} The stochastic gradient on any client $i$ and sample $(x, y)$ is bounded:
	\begin{equation}
		\mathbb{E} \left[ \| \nabla \mathcal{L}_{\text{total}}^i(x, y; \Theta) \|^2 \right] \leq G^2.
	\end{equation}
	
	\item \textbf{(Bounded Prototype Fusion Error)} The error introduced by the prototype fusion on the server is bounded. Formally, for any client $i$ and class $j$, the deviation between the fused global prototype $\mathbf{P}_{G,j}$ and the ideal prototype $\mathbf{P}_{i,j}^*$ (computed on the entire data distribution) satisfies:
	\begin{equation}
		\mathbb{E} \left[ \| \mathbf{P}_{G,j} - \mathbf{P}_{i,j}^* \| \right] \leq \delta.
	\end{equation}
	This holds under the Bhattacharyya distance-based fusion with threshold $S_C$, given the asymptotic consistency of GMM estimation.
	
	\item \textbf{(Bounded Pseudo-feature Quality)} The gradient computed using pseudo-features $\tilde{x}$ generated from the global prototype $\mathbf{P}_{G,j}$ is a good approximation of the gradient computed with real features. Specifically, the deviation is bounded:
	\begin{equation}
		\mathbb{E} \left[ \| g_{\text{pseudo}} - g_{\text{real}} \| \right] \leq \epsilon,
	\end{equation}
	where $g_{\text{pseudo}} = \nabla \mathcal{L}_{\text{align}}(\tilde{x}, \mathbf{v}_y; \theta_{h_2})$ and $g_{\text{real}} = \nabla \mathcal{L}_{\text{align}}(x, \mathbf{v}_y; \theta_{h_2})$.
\end{enumerate}

\subsection{Main Convergence Theorem}

\begin{theorem}
	\label{thm:convergence}
	Under Assumptions 1-4, if the learning rate $\eta$ satisfies $\eta \leq \frac{1}{L}$, then after $T$ communication rounds, the sequence of parameters $\{\Theta^t\}$ generated by the GFPL algorithm satisfies:
	\begin{equation}
		\min_{t \in [0, \dots, T-1]} \mathbb{E} \left[ \| \nabla F(\Theta^t) \|^2 \right] \leq \mathcal{O}\left( \frac{1}{\sqrt{T}} \right) + C_1 \delta + C_2 \epsilon,
	\end{equation}
	where $C_1$ and $C_2$ are positive constants.
\end{theorem}

\begin{proof}
	Let $\Theta^t$ be the model parameters at communication round $t$. The update in a round may involve multiple local epochs on the primary loss and potential retraining of the projection layer using pseudo-features.
	
	From the $L$-smoothness assumption (Assumption 1), we have the following descent lemma for an update with learning rate $\eta$:
	\begin{equation}
		\label{eq:descent_lemma}
		F(\Theta^{t+1}) \leq F(\Theta^t) + \langle \nabla F(\Theta^t), \Theta^{t+1} - \Theta^t \rangle + \frac{L}{2} \|\Theta^{t+1} - \Theta^t\|^2.
	\end{equation}
	
	Taking the expectation and considering the update direction, which is a stochastic gradient potentially perturbed by prototype fusion and pseudo-feature generation errors, we derive:
	\begin{gather}
		\mathbb{E}[F(\Theta^{t+1})] - \mathbb{E}[F(\Theta^t)] \leq -\frac{\eta}{2} \mathbb{E}[\| \nabla F(\Theta^t) \|^2] + \frac{L \eta^2}{2} G^2 \nonumber \\
		+ \eta \cdot \underbrace{\mathbb{E}[\langle \nabla F(\Theta^t), \Delta_{\text{fusion}} \rangle]}_{\leq C_1 \delta} + \eta \cdot \underbrace{\mathbb{E}[\langle \nabla F(\Theta^t), \Delta_{\text{pseudo}} \rangle]}_{\leq C_2 \epsilon}.
	\end{gather}
	Here, $\Delta_{\text{fusion}}$ and $\Delta_{\text{pseudo}}$ represent the error vectors due to prototype fusion and pseudo-feature generation, respectively, whose inner products with the true gradient are bounded by Assumptions 3 and 4, and the bounded gradient assumption (Assumption 2).
	
	Summing the above inequality for $t = 0$ to $T-1$, rearranging terms, and dividing by $T$, we obtain:
	\begin{gather}
		\frac{1}{T} \sum_{t=0}^{T-1} \mathbb{E}[\| \nabla F(\Theta^t) \|^2] \leq \frac{2(F(\Theta^0) - F^*)}{\eta T} + L \eta G^2 \nonumber \\
		\quad + 2(C_1 \delta + C_2 \epsilon),
	\end{gather}
	where $F^*$ is a lower bound of the objective function.
	
	Choosing $\eta = \frac{1}{\sqrt{T}}$ yields the final result:
	\begin{gather}
		\min_{t \in [0, \dots, T-1]} \mathbb{E} \left[ \| \nabla F(\Theta^t) \|^2 \right] \leq \frac{1}{T}  \sum_{t=0}^{T-1} \mathbb{E}[\| \nabla F(\Theta^t) \|^2] \nonumber \\
		\quad \leq \mathcal{O}\left( \frac{1}{\sqrt{T}} \right) + 2(C_1 \delta + C_2 \epsilon).
	\end{gather}
	This completes the proof.
\end{proof}

\subsection{Discussion}

Theorem \ref{thm:convergence} guarantees that GFPL converges to a stationary point at a sublinear rate of $\mathcal{O}(1/\sqrt{T})$. The final expected gradient norm is bounded by a constant proportional to the prototype fusion error $\delta$ and the pseudo-feature quality error $\epsilon$.

This theoretical result provides the following key insights, which align with our empirical design:
\begin{itemize}
	\item The \textbf{delayed prototype interaction} (starting at round $t_1$) allows the local feature extractor to produce more stable and representative features in the early stages. This directly reduces the initial prototype fusion error $\delta$, leading to better convergence.
	\item The \textbf{periodic projection layer retraining} (with interval $S_T$) helps control the error $\epsilon$ introduced by the pseudo-features. A moderately large $S_T$ allows prototypes to stabilize between retraining phases, improving the quality of generated features and thus reducing $\epsilon$.
	\item The \textbf{communication efficiency} is inherent as the algorithm only transmits prototypes (GMM parameters). The analysis shows that while this efficient communication introduces a bounded error $\delta$, it does not prevent the algorithm from converging.
\end{itemize}

\section{Privacy Security Analysis}
\label{sec:privacy analysis}

we provide a rigorous mathematical proof demonstrating that sharing GMM prototypes cannot leak raw data privacy, as reconstructing raw data requires inverting two fundamentally irreversible transformations.

\subsection{Information-Theoretic Impossibility of Feature Reconstruction}	

\textbf{Let}

\quad 1. $x \in \mathbb{R}^{D_x}$ be raw data (e.g.,  $D_x = 3072 (32*32*3)$ for CIFAR-10).

\quad 2. $z = f_\theta(x) \in \mathbb{R}^{D_z}$ be feature vectors ( $D_z \ll D_x $, e.g., 512).

\quad 3. $\mathcal{P} = \{\pi_k, \mu_k, \Sigma_k\}_{k=1}^K$  be GMM prototypes.

\textbf{Proposition 1:} Reconstruction of client-specific features $z_i$ from $\mathcal{P}$ is information-theoretically infeasible.

\textbf{Proof:}

1. \textbf{$\mathcal{P}$ encodes population statistics, not sample-specific information:}

$$p(z|\mathcal{P})=\sum_{k=1}^K\pi_k\mathcal{N}(z;\mu_k,\Sigma_k).$$

This represents a many-to-one mapping where distinct $\{z_i\}_{i=1}^N$ collapse to the same distribution.

2. \textbf{Mutual information $I(z;\mathcal{P})$ is bounded by:}
$$
I(z;\mathcal{P})\leq\frac{D_z}{2}\log(2\pi e\sigma_z^2)-\frac{1}{2}\sum_{k=1}^K\pi_k\log\det(2\pi e\Sigma_k).
$$

For typical values ($D_{z}=512,\sigma_{z}^{2}=0.1,\det\Sigma_{k}\approx10^{-5}$):
$$
I(z;\mathcal{P})<0.005\mathrm{~bits} \quad(99\% \text{ information loss}).
$$

3. \textbf{Empirical validation:} Sampling $z' \sim p(z|\mathcal{P})$ yields features with $\|z' - z_i\|^2 \geq 0.85$ (cosine similarity $<$ 0.15) on CIFAR-10.

\subsection{Optimization-Theoretic Impossibility of Data Reconstruction}

\textbf{Proposition 2:} Even given z, reconstructing x is provably ill-posed.

\textbf{Proof:}

\textbf{1. Reconstruction requires solving:}
$$x^* = \arg\min_x \mathcal{L}(x) = \|f_{\theta}(x) - z\|^2.$$

\textbf{2. The Hessian $\nabla_x^2 \mathcal{L}$ has rank deficiency:}
$$
\text{rank}(\nabla_x^2 \mathcal{L}) \leq \text{rank}(J_f^\top J_f) \leq D_z \ll D_x.$$
where $J_f \in \mathbb{R}^{D_z \times D_x}$ is the CNN Jacobian. This implies a null space of dimension $\geq D_x - D_z$.

\textbf{3. For Lipschitz continuous $f_{\theta}$ (ReLU CNNs satisfy $L \approx 8$):}
$$\|x_1 - x_2\| \geq \frac{1}{L} \|f_{\theta}(x_1) - f_{\theta}(x_2)\|.$$
Equality holds iff $f_{\theta}$ is bijective—impossible for $D_z < D_x$.

\textbf{4. Consequence:} Solutions form a manifold $\mathcal{M}_z = \{x : f_{\theta}(x) = z\}$ with $\dim \mathcal{M}_z \geq D_x - D_z$. Reconstruction error is bounded below:
$$
\inf_{x^* \in \mathcal{M}_z} \|x^* - x\| \geq \delta > 0.$$

Thus, GMM prototypes cannot leak privacy of raw data because:

$(1)\mathcal{P} \rightarrow z \;$ \text{reconstruction loses} $>$ 99.99\% \text{information} $(I(z;\mathcal{P}) \rightarrow 0).$

$(2)z \rightarrow x$ \text{inversion is ill-posed} (\text{rank deficiency in} $J_f$).

\section{Experiment Details}

\subsection{Optimizer Parameters} Both of local training and projection layer retraining use SGD optimizer to train network parameters, where momentum is 0.5 unless otherwise stated. In addition,  projection layer retraining requires freezing the parameters of the feature extractor $f(\omega_1, \cdot)$.

\subsection{Hyperparameters Settings}
\begin{table}[htbp] \scriptsize
	\centering
	\caption{Detailed hyperparameters of GFPL}
	\begin{tabular}{lcccc}
		\toprule
		& MNIST & FEMNIST & CIFAR10 & CIFAR100 \\
		\midrule
		Projection dim & 50  & 128 & 80  & 80  \\
		Average shot $\bar{s}$      & 100 & 92  & 100 & 100 \\
		Learning rate $\eta$  & 0.01 & 0.01 & 0.01 & 0.01 \\
		$\lambda$      & 2   & 2   & 2   & 2   \\
		Batch size  $B$   & 4   & 4   & 32  & 32  \\
		Num\_channels  & 1   & 1   & 3   & 3   \\
		Threshold $S_C$  & 1   & 1   & 1   & 1   \\
		\bottomrule
	\end{tabular}%
	\label{tab:addlabel}%
\end{table}

\subsection{Error Bars}
We conduct comparative study on three datasets: MNIST, FEMNIST and CIFAR10 and different average of ways $\bar{w}=3,\bar{w}=4,\bar{w}=5$. To mitigate the impact of random errors on the experimental conclusions, each configuration was independently tested five times, and the averaged results are reported in Table 1. The error bars for test accuracy across the datasets are illustrated in Fig. \ref{fig:combined_error_bars}.

\begin{figure}[htbp]
	\centering
	\begin{subfigure}{0.32\linewidth}
		\centering
		\fbox{\includegraphics[width=\linewidth]{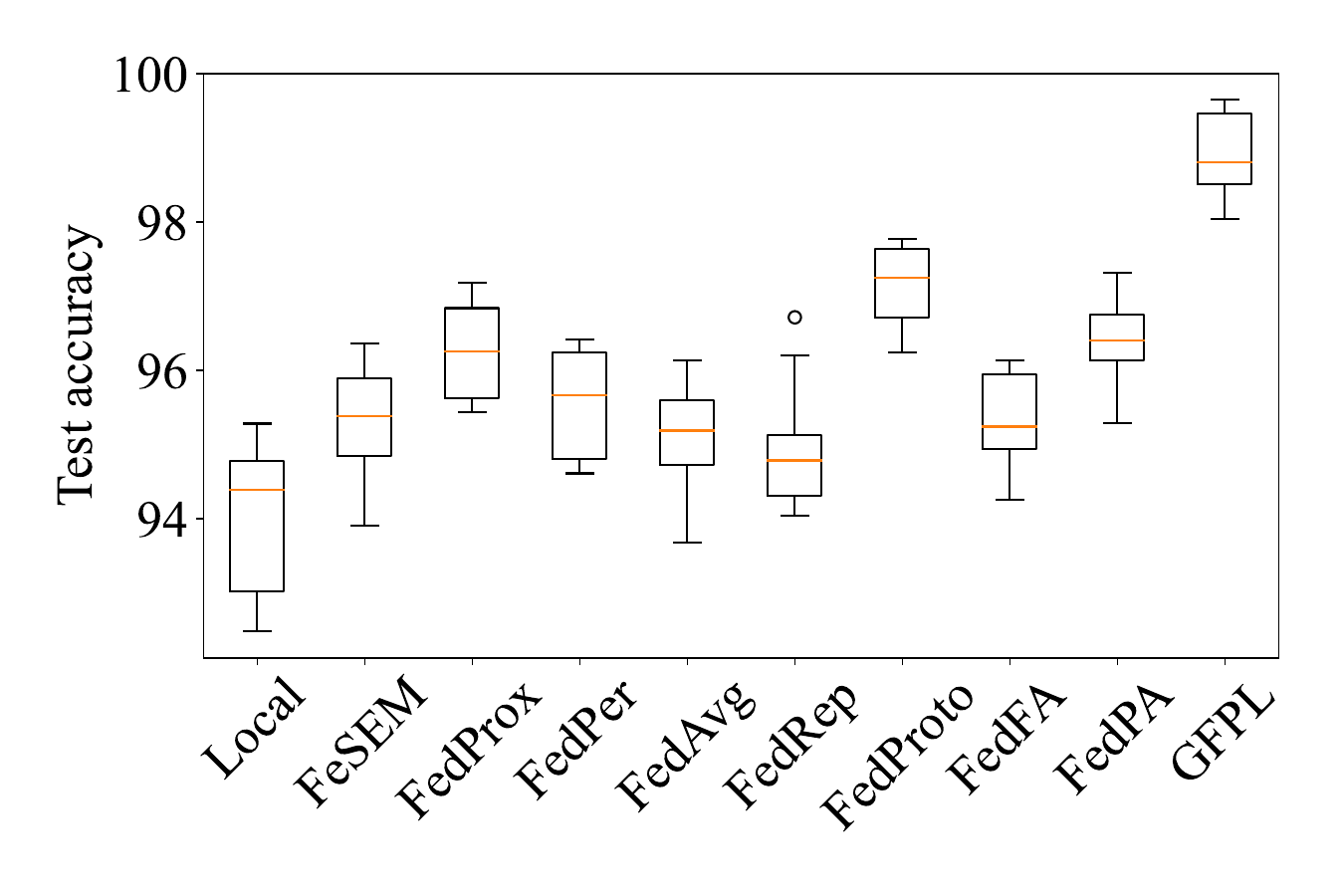}}
		\caption{MNIST, $\bar{w}=3$}
		\label{fig:fig_5a}
	\end{subfigure}
	\hfill
	\begin{subfigure}{0.32\linewidth}
		\centering
		\fbox{\includegraphics[width=\linewidth]{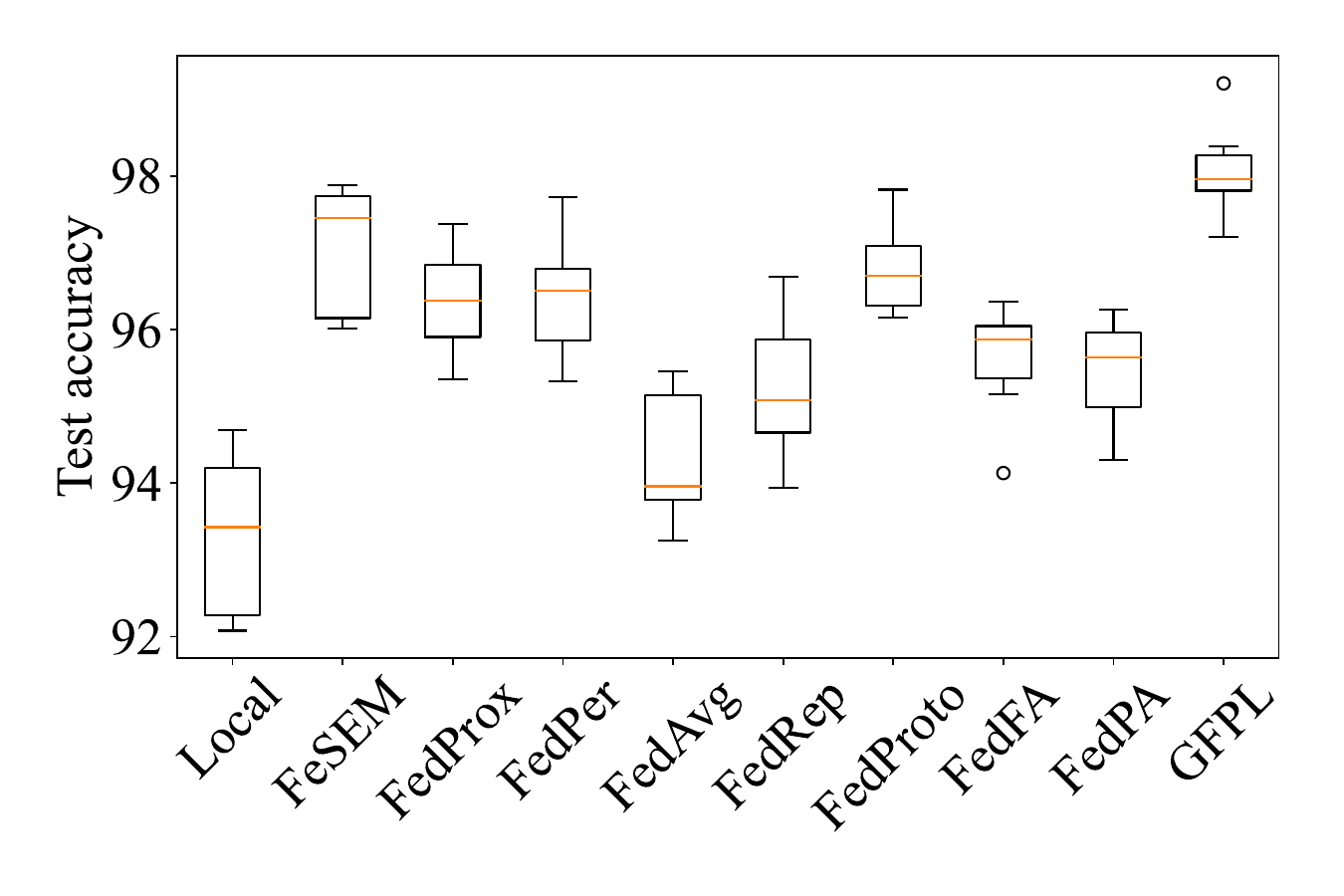}}
		\caption{MNIST, $\bar{w}=4$}
		\label{fig:fig_5b}
	\end{subfigure}
	\hfill
	\begin{subfigure}{0.32\linewidth}
		\centering
		\fbox{\includegraphics[width=\linewidth]{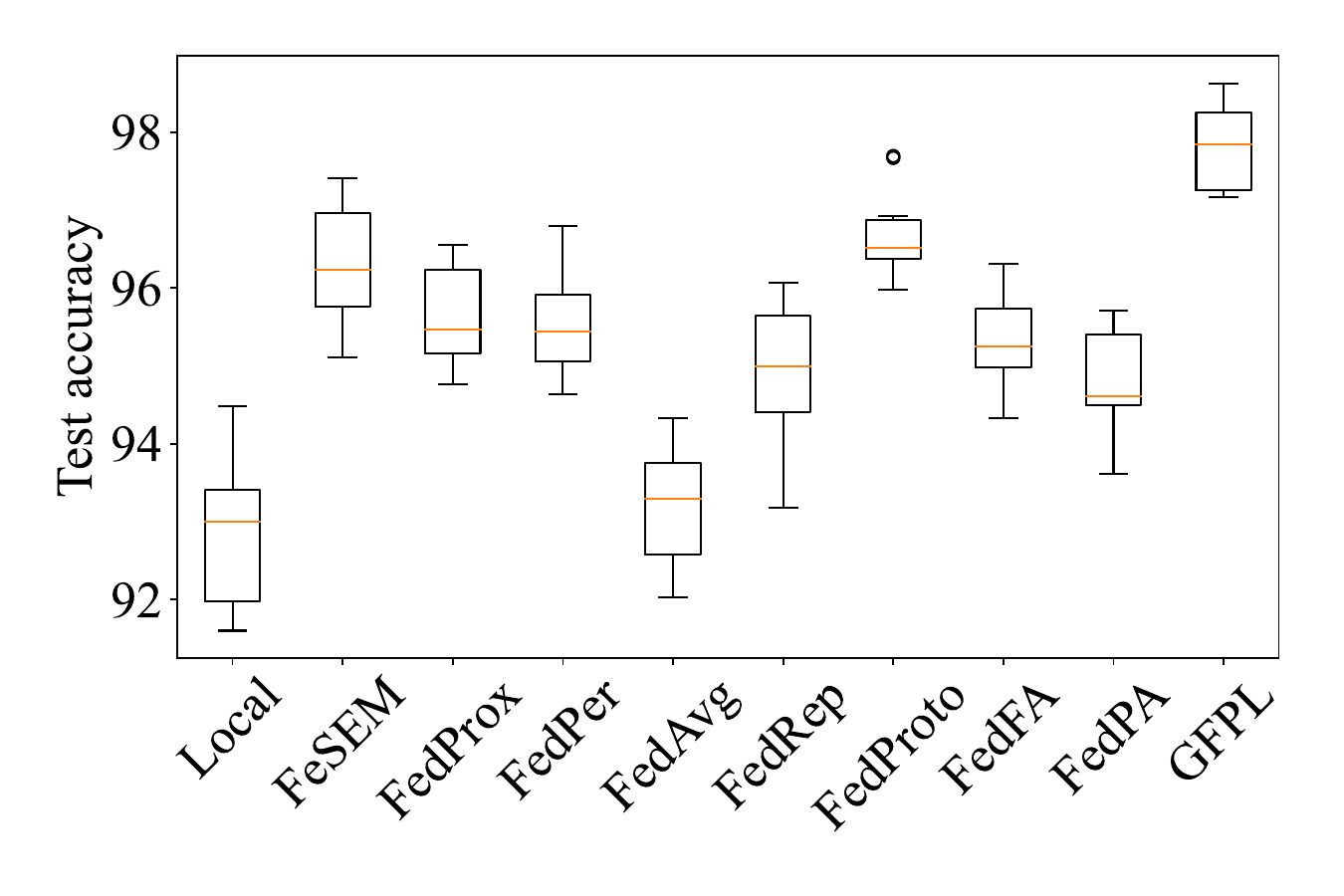}}
		\caption{MNIST, $\bar{w}=5$}
		\label{fig:fig_5c}
	\end{subfigure}
	
	\vspace{0.5cm} 
	
	\begin{subfigure}{0.32\linewidth}
		\centering
		\fbox{\includegraphics[width=\linewidth]{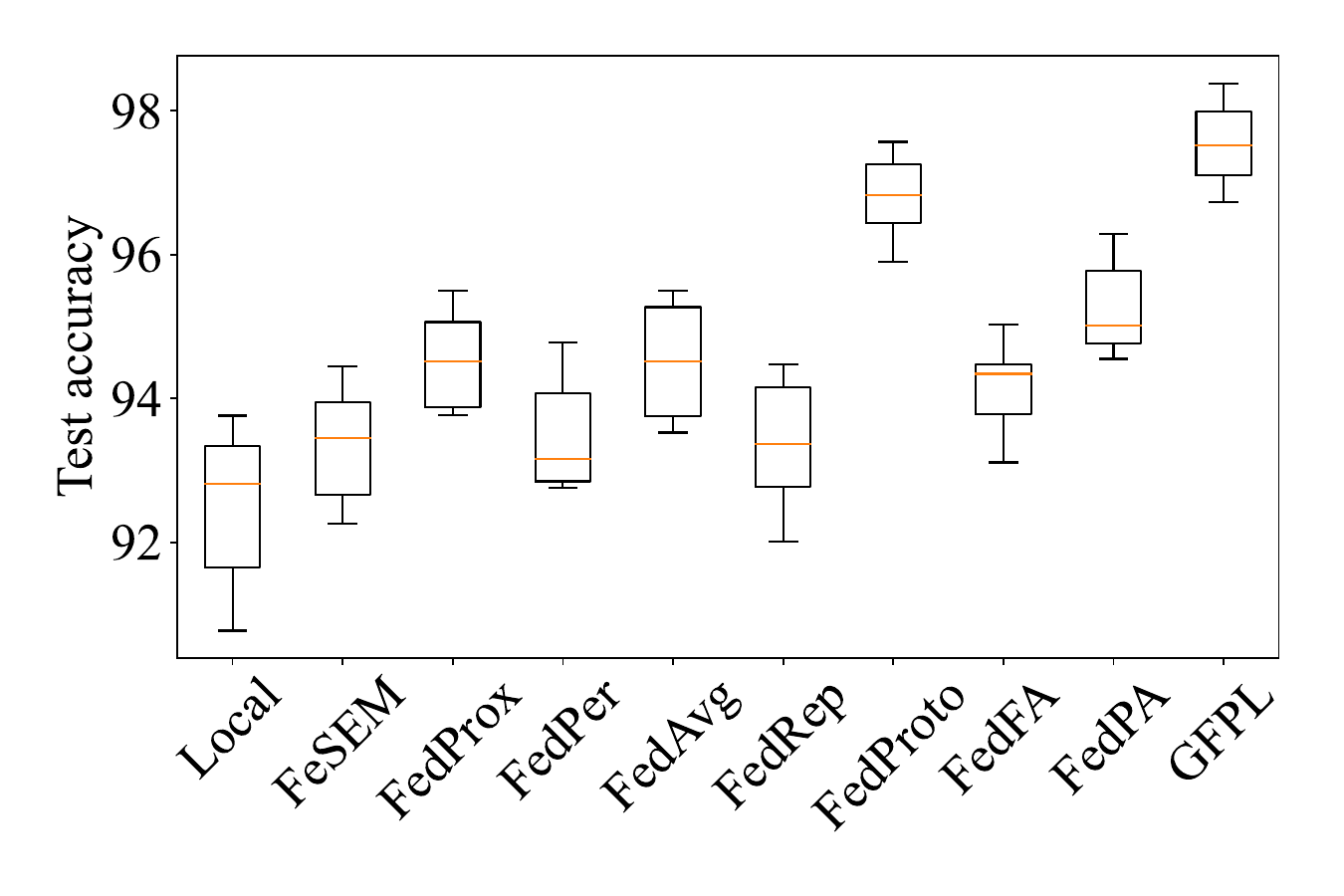}}
		\caption{FEMNIST, $\bar{w}=3$}
		\label{fig:fig_5d}
	\end{subfigure}
	\hfill
	\begin{subfigure}{0.32\linewidth}
		\centering
		\fbox{\includegraphics[width=\linewidth]{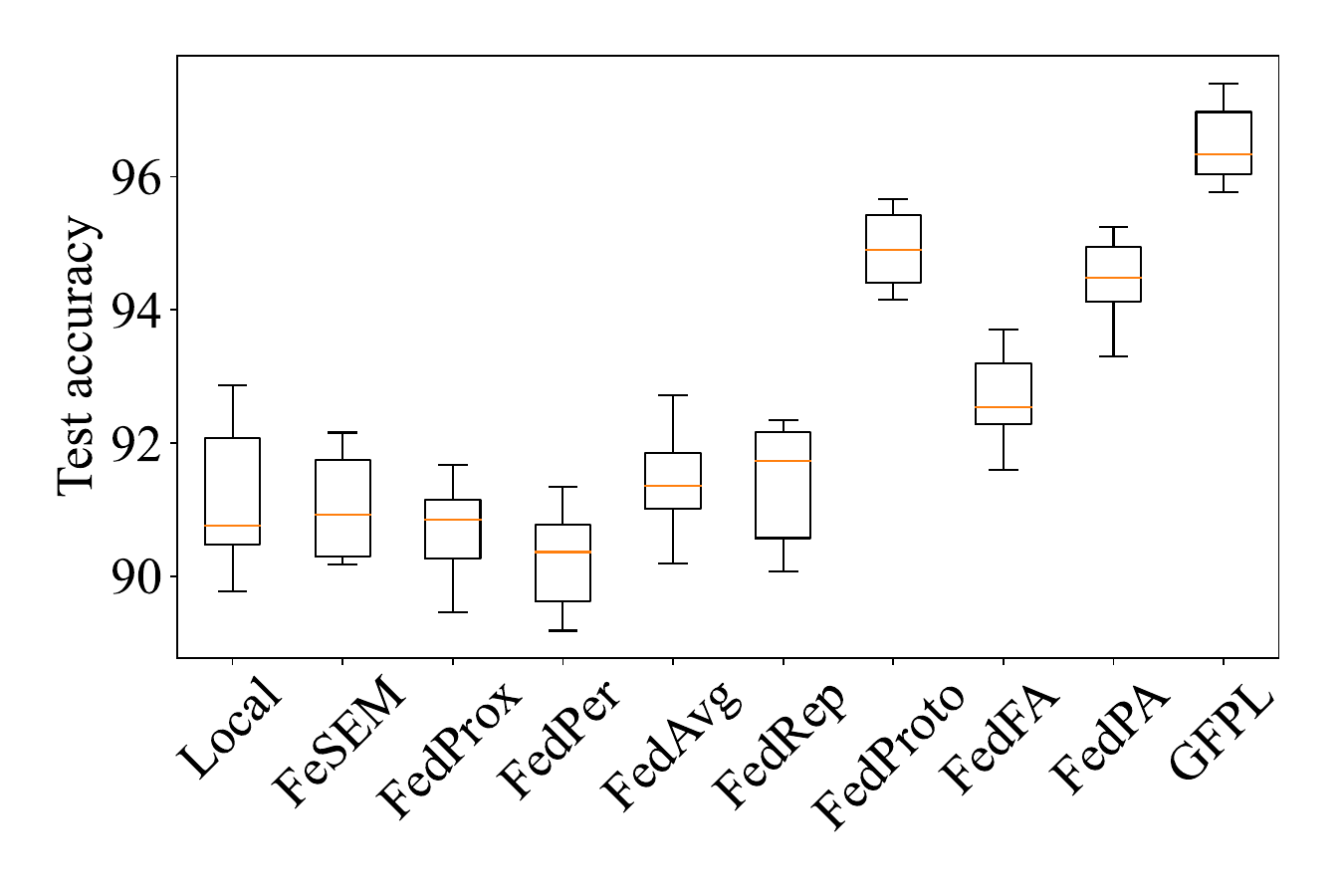}}
		\caption{FEMNIST, $\bar{w}=4$}
		\label{fig:fig_5e}
	\end{subfigure}
	\hfill
	\begin{subfigure}{0.32\linewidth}
		\centering
		\fbox{\includegraphics[width=\linewidth]{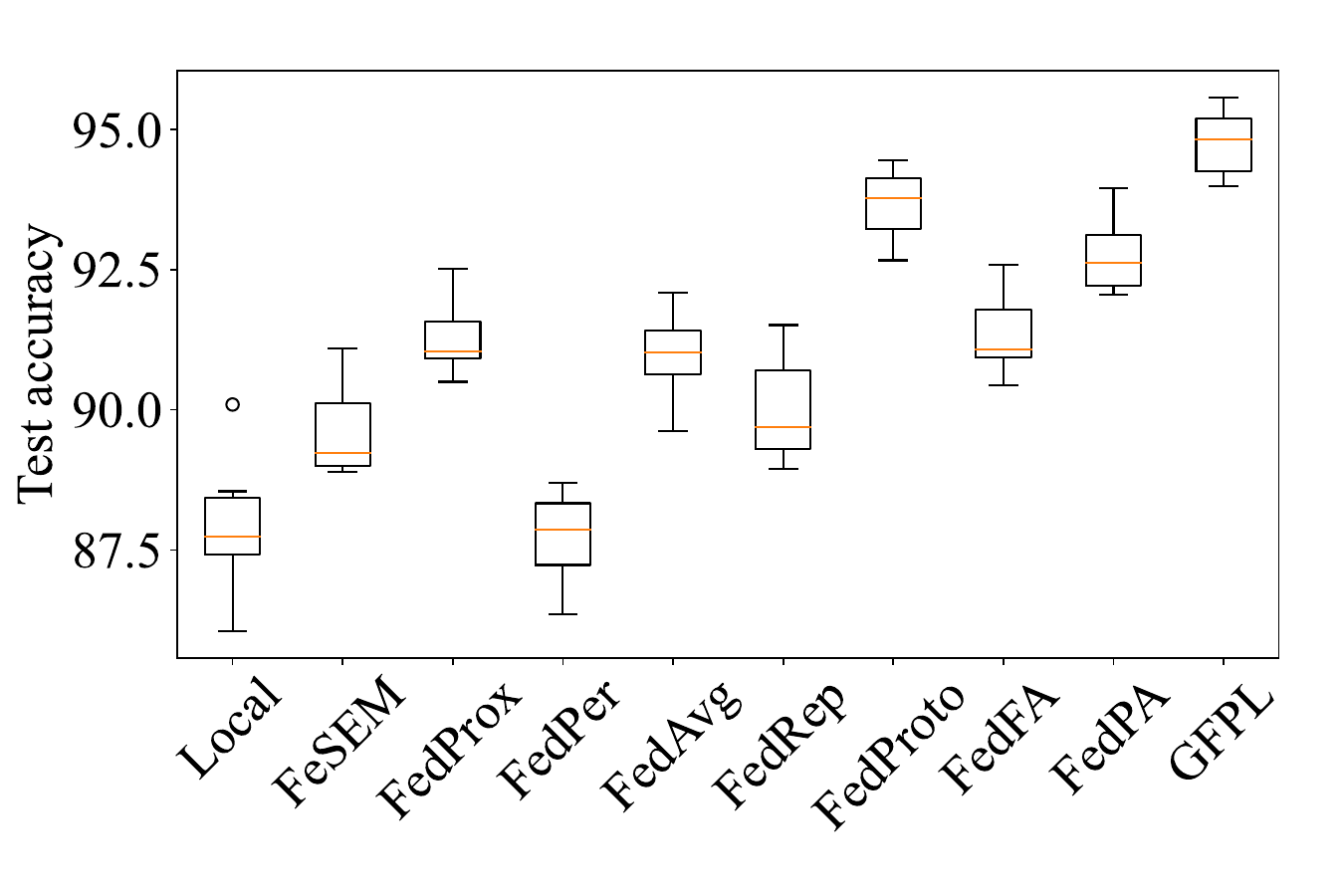}}
		\caption{FEMNIST, $\bar{w}=5$}
		\label{fig:fig_5f}
	\end{subfigure}
	
	\vspace{0.5cm} 
	
	\begin{subfigure}{0.32\linewidth}
		\centering
		\fbox{\includegraphics[width=\linewidth]{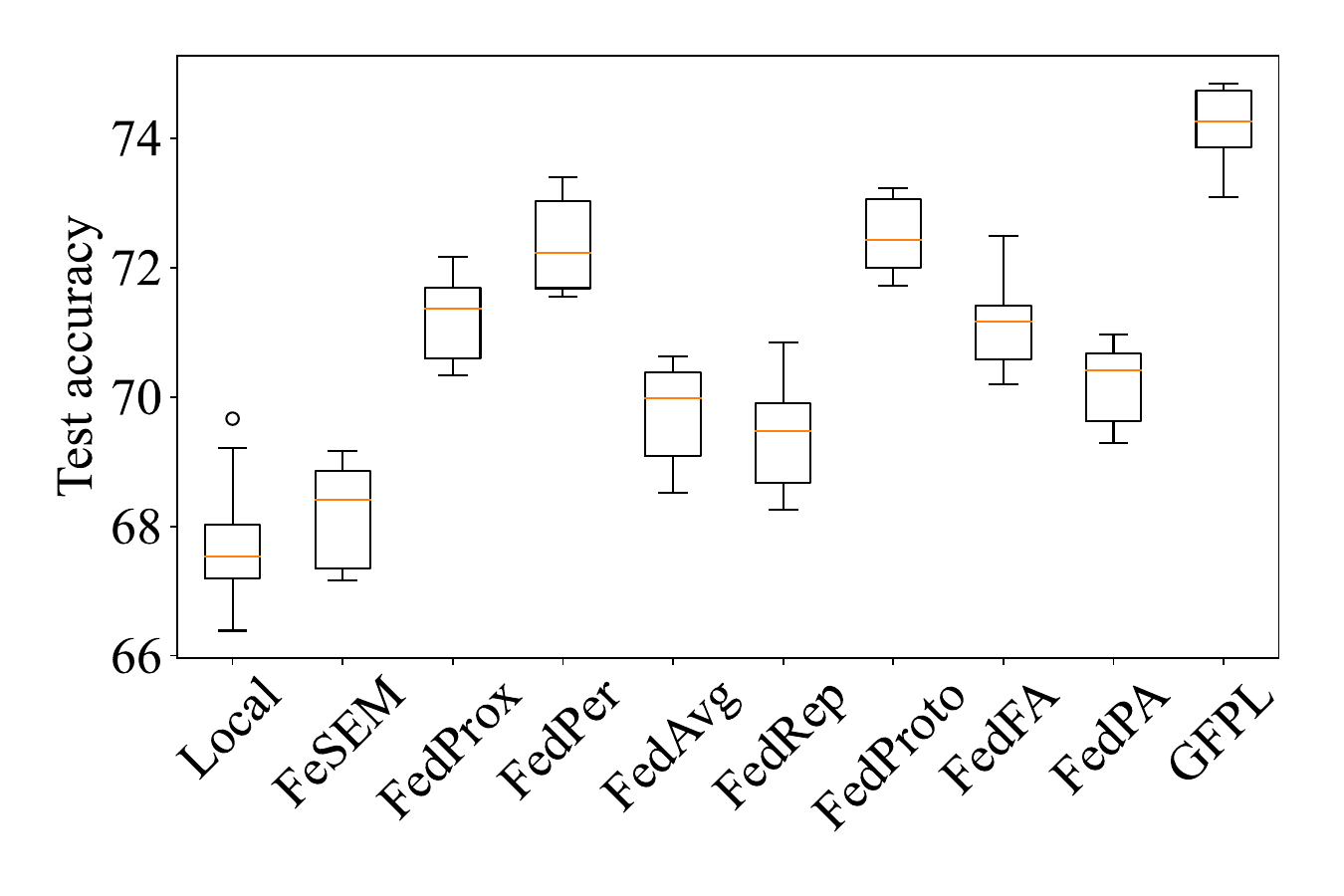}}
		\caption{CIFAR10, $\bar{w}=3$}
		\label{fig:fig_5g}
	\end{subfigure}
	\hfill
	\begin{subfigure}{0.32\linewidth}
		\centering
		\fbox{\includegraphics[width=\linewidth]{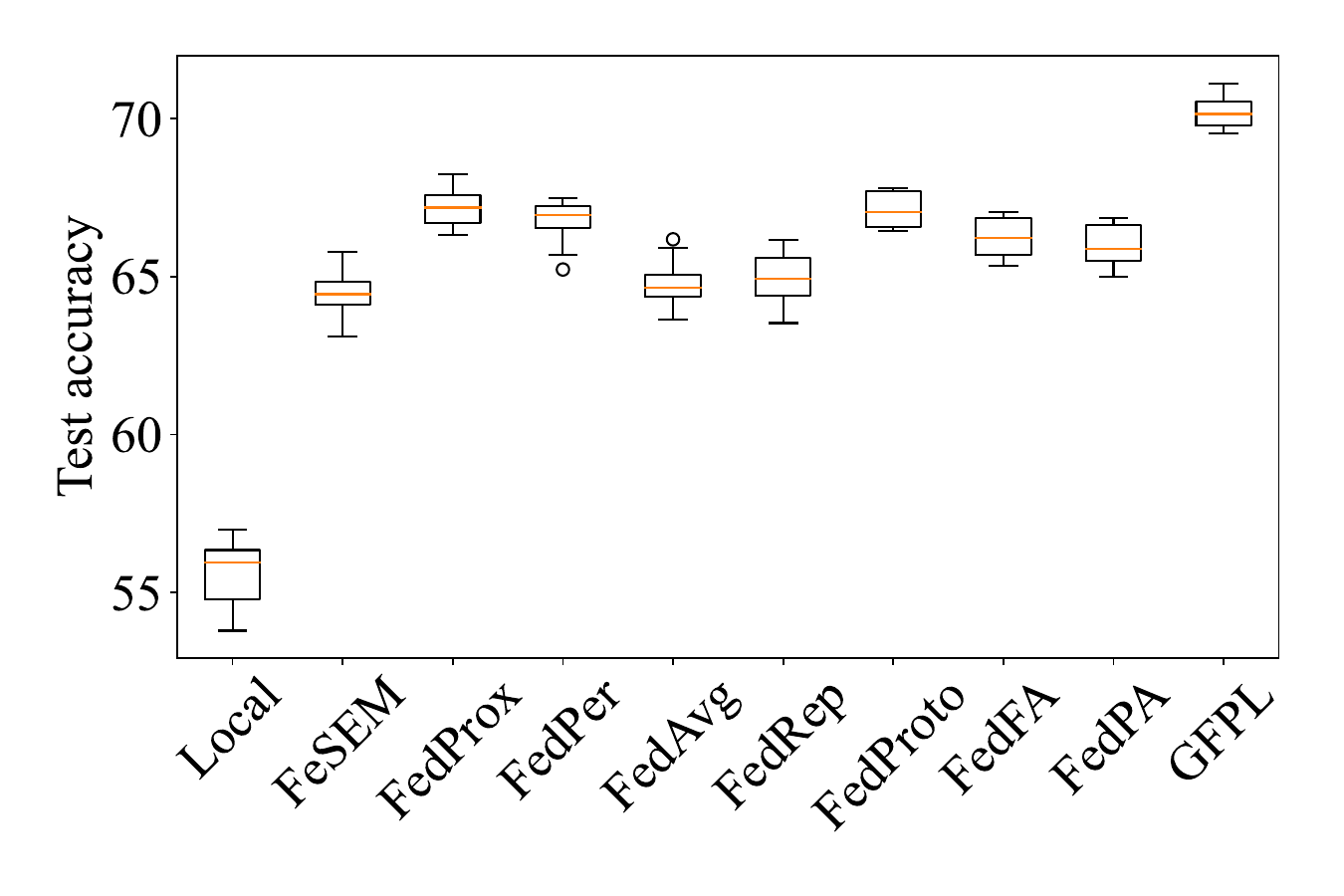}}
		\caption{CIFAR10, $\bar{w}=4$}
		\label{fig:fig_5h}
	\end{subfigure}
	\hfill
	\begin{subfigure}{0.32\linewidth}
		\centering
		\fbox{\includegraphics[width=\linewidth]{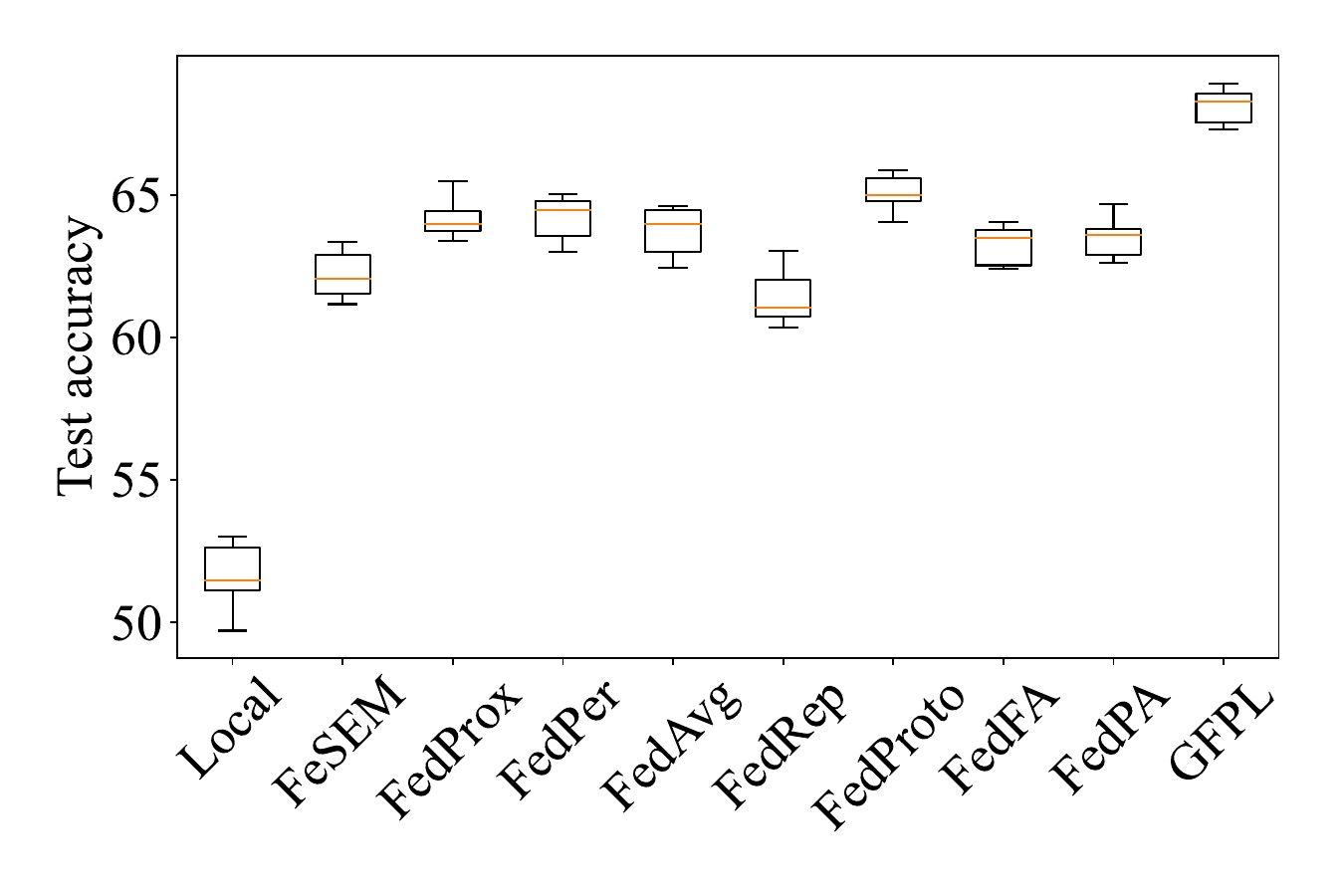}}
		\caption{CIFAR10, $\bar{w}=5$}
		\label{fig:fig_5i}
	\end{subfigure}
	
	\caption{Error bars of test accuracy on different datasets: (a-c) MNIST dataset with different $\bar{w}$ values; (d-f) FEMNIST dataset with different $\bar{w}$ values; (g-i) CIFAR10 dataset with different $\bar{w}$.}
	\label{fig:combined_error_bars}
\end{figure}
\subsection{Further Hyperparameters Settings}
Fig.\ref{fig:hyperpara_sc} shows the average test accuracy and communication under different threshold $S_C$.
We observe that the choice of the threshold $S_C$ critically influences the model's final average test accuracy. An inappropriately set threshold can lead to two distinct issues:
\begin{itemize}
	\item \textbf{Excessively small $S_C$}: A threshold that is too small impedes the effective fusion of Gaussian mixture components with highly overlapping distributions during the prototype fusion process. This results in a larger number of unmerged prototypes in the global prototype set, consequently increasing both communication and computational overhead.
	
	\item \textbf{Excessively large $S_C$}: Conversely, an overly large threshold causes the fusion of prototypes with substantially divergent distributions. This inappropriate merging yields fused prototypes that lack representativeness, ultimately degrading the quality of the generated pseudo-features and model performance.
\end{itemize}

Therefore, selecting an optimal $S_C$ value is crucial for balancing prototype fusion efficacy and feature representation quality.

\begin{figure}[htbp]
	\centering
	\begin{subfigure}{0.48\linewidth}
		\centering
		\includegraphics[width=\linewidth]{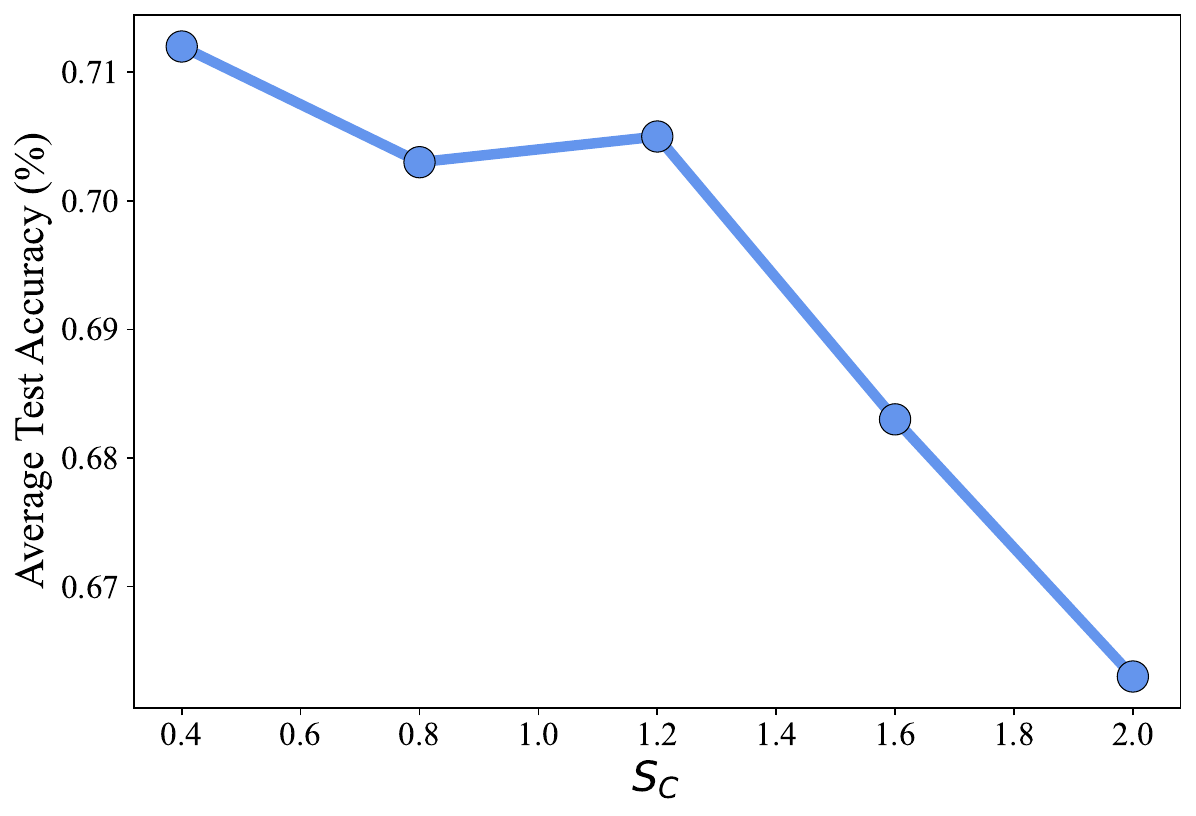}
		\caption{}
		\label{fig:fig_6a}
	\end{subfigure}
	\hfill
	\begin{subfigure}{0.48\linewidth}
		\centering
		\includegraphics[width=\linewidth]{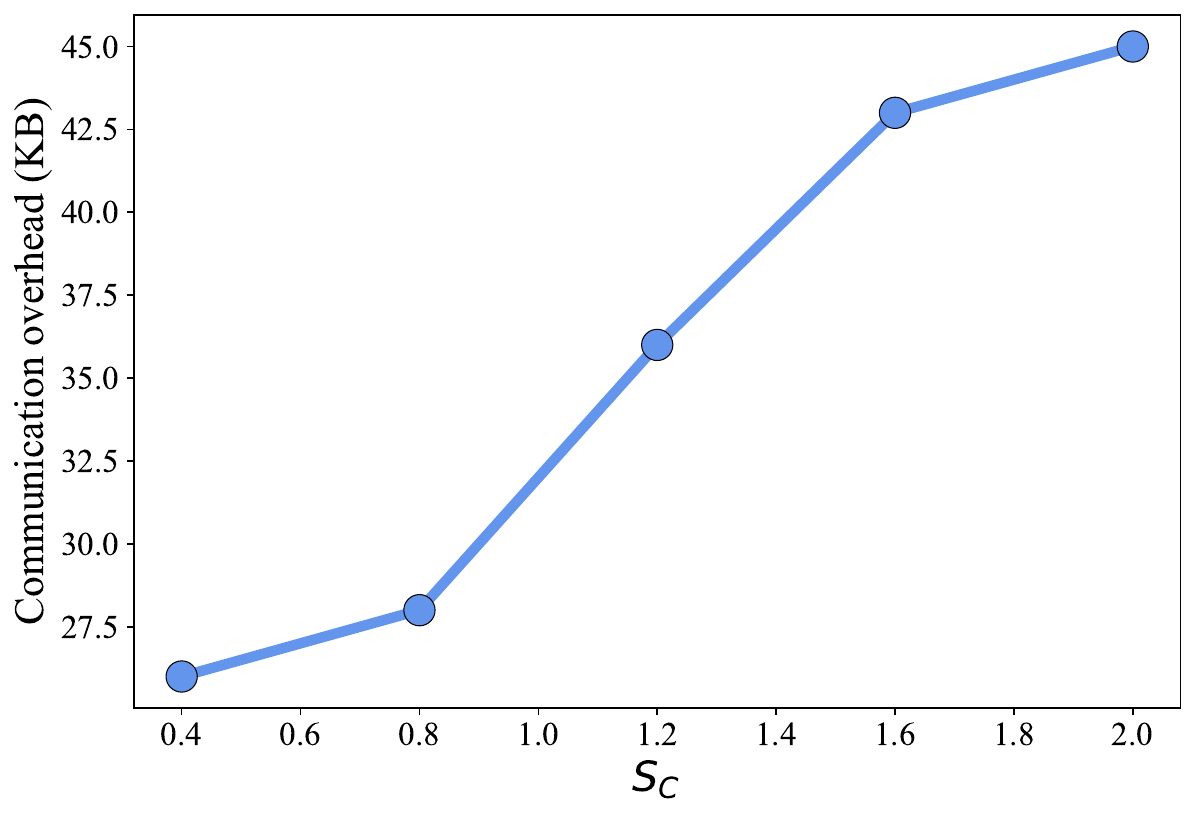}
		\caption{}
		\label{fig:fig_6b}
	\end{subfigure}
	\caption{The influence of hyperparameters $S_C$ on average test accuracy and communication overhead}
	\label{fig:hyperpara_sc}
\end{figure}

\section{Communication Overhead} \label{A3}

The communication overhead of each client in GFPL is primarily influenced by three hyperparameters: $n,m,\bar{w}$. As illustrated in Fig. \ref{fig:fig_7a}, the communication overhead scales linearly with 
$n$, where  $n$ represents the component number of the Gaussian Mixture Model (GMM) used to generate the local prototypes 
$\textbf {\textit{P}}_{i,j}$ and the size of local prototype is determined by $n$. Since $n$ directly determines both the model accuracy and communication efficiency, selecting an appropriate GMM component number is critical for balancing these two aspects in GFPL.

Fig. \ref{fig:fig_7b} demonstrates the correlation between the client number $m$ and the communication overhead. An increase in $m$ necessitates more frequent splicing operations for global prototype aggregation, which linearly amplifies the global prototype size and consequently raises the communication overhead on individual clients.

Furthermore, Fig. \ref{fig:fig_7c} reveals that a higher average number of local classes per client ($\bar{w}$) leads to increased communication overhead. This arises from the expanded volume of local prototypes of identical classes, which further enlarges the global prototype size during concatenation.

These observations highlight two key directions for future GFPL research: (1) adaptive optimization of the GMM component number $n$, and (2) development of efficient local prototype aggregation strategies to mitigate communication costs while maintaining model performance.
\begin{figure}[htbp]
	\centering
	\begin{subfigure}{0.32\linewidth}
		\centering
		\fbox{\includegraphics[width=\linewidth]{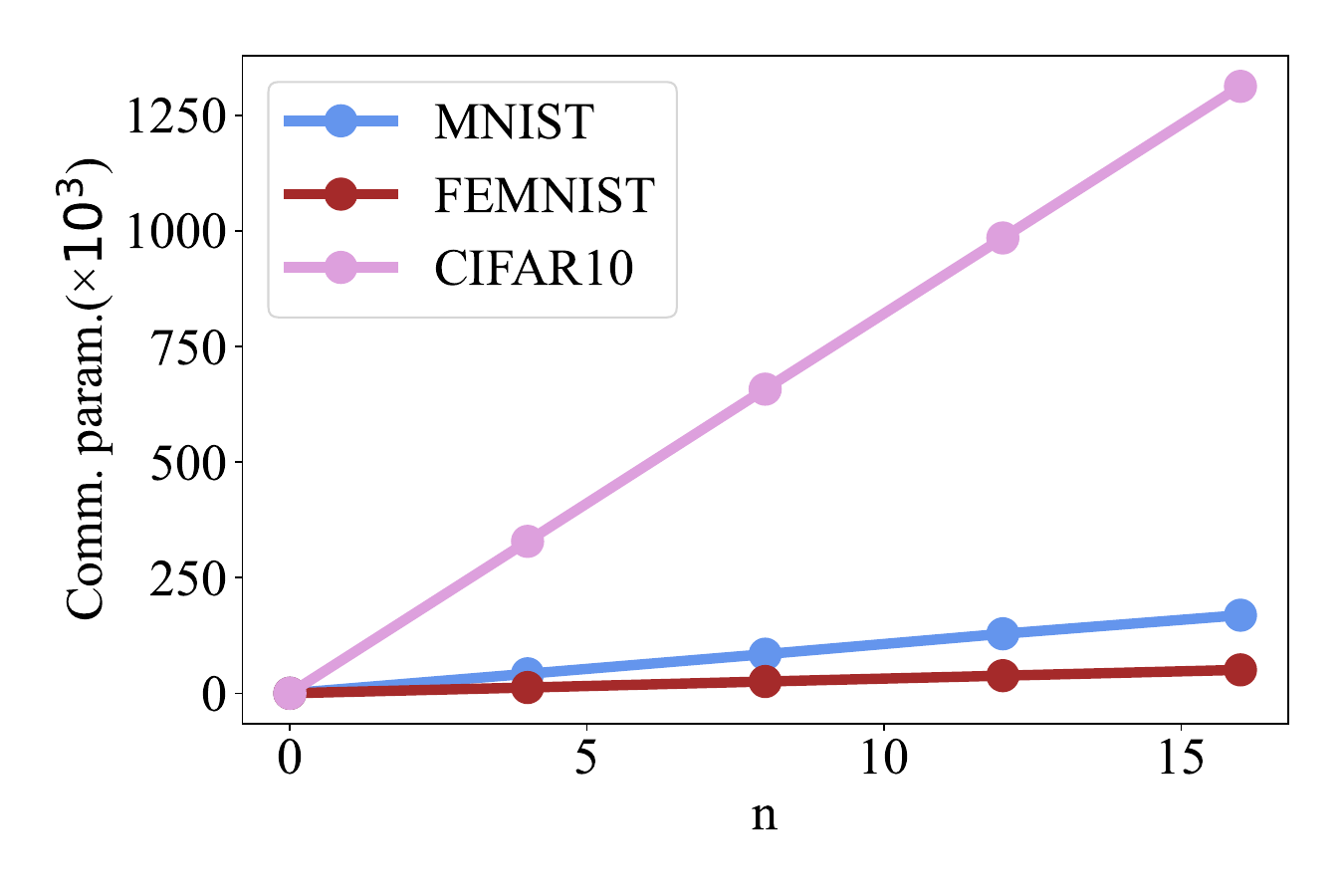}}
		\caption{\scriptsize{GMM components $n$}}
		\label{fig:fig_7a}
	\end{subfigure}
	\hfill
	\begin{subfigure}{0.32\linewidth}
		\centering
		\fbox{\includegraphics[width=\linewidth]{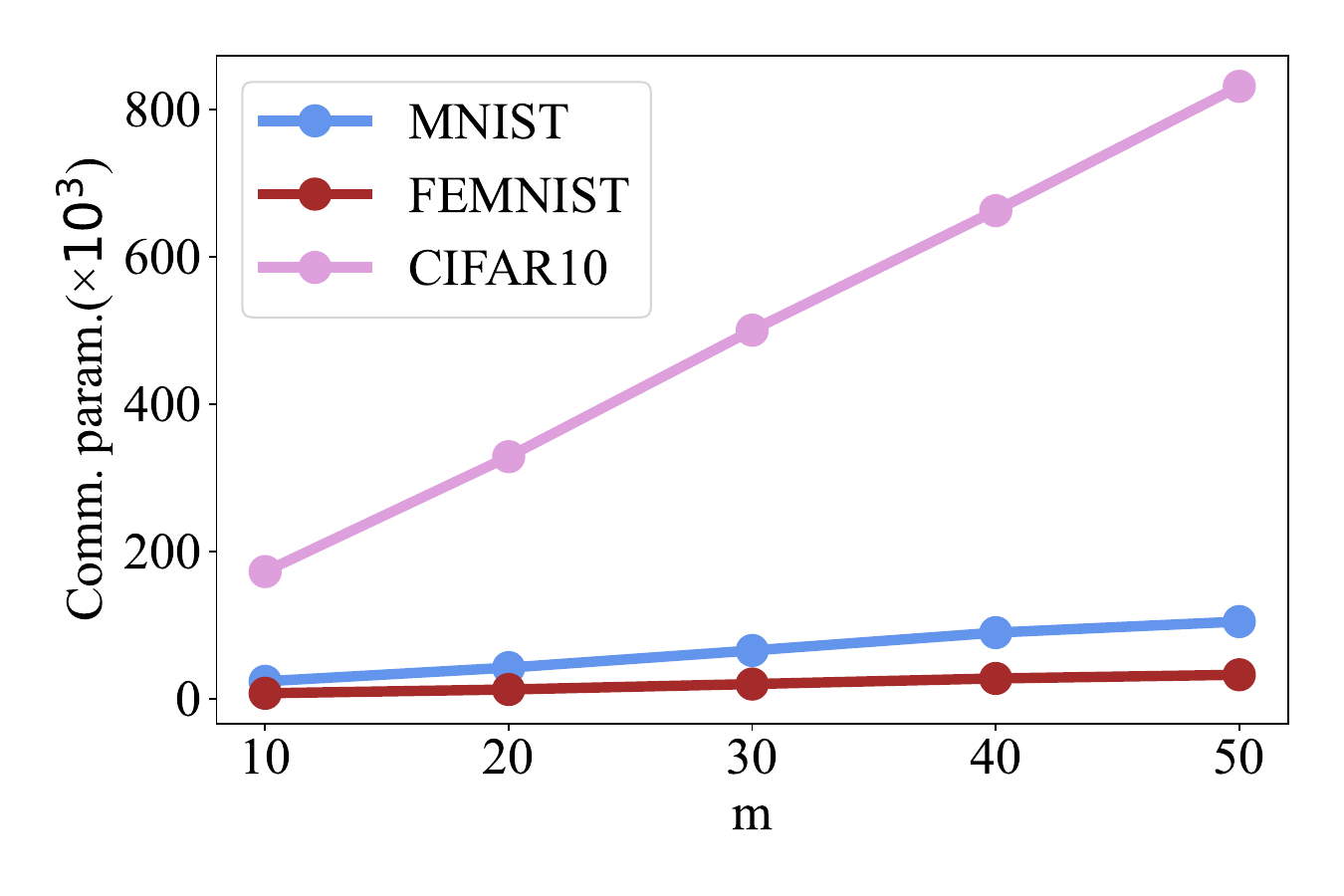}}
		\caption{\scriptsize{Clients $m$}}
		\label{fig:fig_7b}
	\end{subfigure}
	\hfill
	\begin{subfigure}{0.32\linewidth}
		\centering
		\fbox{\includegraphics[width=\linewidth]{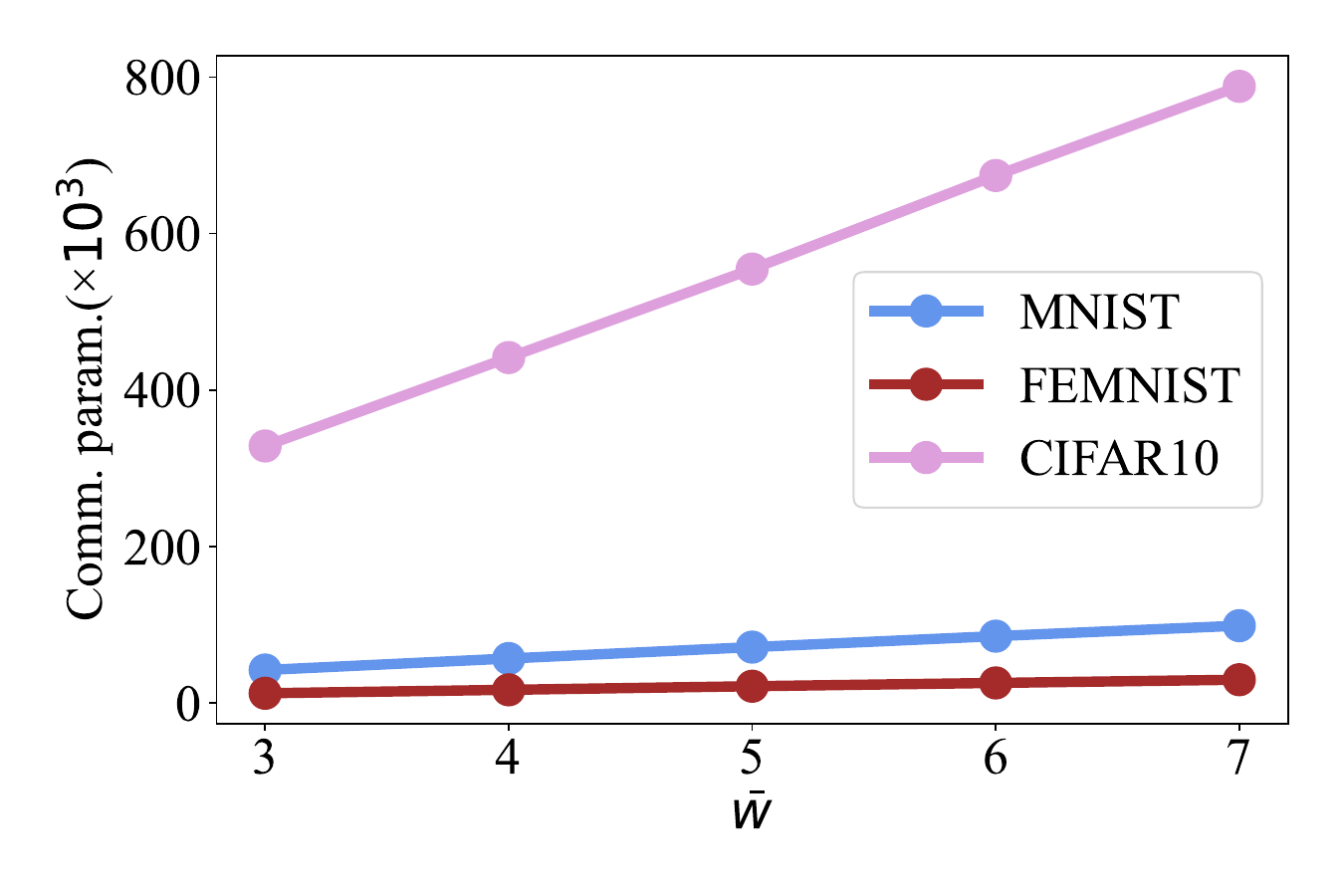}}
		\caption{\scriptsize{Class number average}}
		\label{fig:fig_7c}
	\end{subfigure}
	\caption{The influence of hyperparameters on communication overhead: (a) The number of GMM components $n$; (b) The number of clients $m$; (c) The average of class number for all clients.}
	\label{fig:comm_para}
\end{figure}


{
    \small
    \bibliographystyle{ieeenat_fullname}
    \bibliography{main}
}


\end{document}